\documentclass[letterpaper]{article} % DO NOT CHANGE THIS
\usepackage{aaai24}  % DO NOT CHANGE THIS
\usepackage{times}  % DO NOT CHANGE THIS
\usepackage{helvet}  % DO NOT CHANGE THIS
\usepackage{courier}  % DO NOT CHANGE THIS
\usepackage[hyphens]{url}  % DO NOT CHANGE THIS
\usepackage{graphicx} % DO NOT CHANGE THIS
\usepackage{natbib}  % DO NOT CHANGE THIS AND DO NOT ADD ANY OPTIONS TO IT
\usepackage{caption} % DO NOT CHANGE THIS AND DO NOT ADD ANY OPTIONS TO IT
\usepackage{amsmath,amssymb,amsfonts,array}
\usepackage{algorithmic}
\usepackage[noline,ruled,algo2e,linesnumbered]{algorithm2e}
\usepackage{caption}
\usepackage{subcaption}
\usepackage{graphicx}

\SetKwIF{If}{ElseIf}{Else}{if}{then}{else if}{else}{end if}%

\newcommand{\cX}{\mathcal{X}} % 
\newcommand{\cY}{\mathcal{Y}} % 
\newcommand{\No}{{N_o}} %

\DeclareMathOperator*{\argmin}{arg\,min}

\DeclareMathOperator*{\GD}{\text{GD}^+}
\DeclareMathOperator*{\IGD}{\text{IGD}^+}

\newcommand{\lcb}{\texttt{LCB}\xspace} % 
\usepackage[english]{babel}
\usepackage{amsthm}
\newtheorem{claim}{Claim}
\newtheorem{definition}{Definition}
\newtheorem{corollary}{Corollary}

\usepackage{cleveref}
\usepackage{color}
\usepackage{marginnote}

\usepackage{todonotes}
\title{Parallel Multi-Objective Hyperparameter Optimization with Uniform Normalization and Bounded Objectives}
\author{
    Romain Egele\textsuperscript{\rm 1,2,}\equalcontrib, Tyler Chang\textsuperscript{\rm 2,}\equalcontrib, Yixuan Sun\textsuperscript{\rm 2},\\
    Venkatram Vishwanath\textsuperscript{\rm 2}, Prasanna Balaprakash\textsuperscript{\rm 3}
}
\affiliations{
    \textsuperscript{\rm 1}Universit\'e Paris-Saclay (France),
    \textsuperscript{\rm 2}Argonne National Laboratory (USA)\\
    \textsuperscript{\rm 3}Oak Ridge National Laboratory (USA)\\
    romain.egele@universite-paris-saclay.fr
}

\begin{document}

\maketitle

\begin{abstract} % max 2,000 characters
Machine learning (ML) methods offer a wide range of configurable hyperparameters that have a significant influence on their performance. While accuracy is a commonly used performance objective, in many settings, it is not sufficient. Optimizing the ML models with respect to multiple objectives such as accuracy, confidence, fairness, calibration, privacy, latency, and memory consumption is becoming crucial. To that end, hyperparameter optimization, the approach to systematically optimize the hyperparameters, which is already challenging for a single objective, is even more challenging for multiple objectives. In addition, the differences in objective scales, the failures, and the presence of outlier values in objectives make the problem even harder. We propose a multi-objective Bayesian optimization (MoBO) algorithm that addresses these problems through uniform objective normalization and randomized weights in scalarization. We increase the efficiency of our approach by imposing constraints on the objective to avoid exploring unnecessary configurations (e.g., insufficient accuracy). Finally, we leverage an approach to parallelize the MoBO which results in a 5x speed-up when using 16x more workers.
\end{abstract}

% Submission guidelines: https://aaai.org/aaai-conference/aaai-24-call-for-proposals/

\section{Introduction}
\label{sec:introduction}

%%% General context of the work: Hyperparameter Optimization
Machine learning (ML) libraries (e.g., Scikit-Learn, Torch, Keras, Tensorflow) now offer a variety of data processing and modeling algorithms that can be combined to build complex workflows. Oftentimes, such learning workflows expose a large number of hyperparameters which gives the possibility to customize the workflow (e.g., optimizer settings, neural architecture, data processing, etc...). However, tuning a large number of hyperparameters is manually intractable  due to its combinatorial nature. Therefore, many algorithms have been developed to tackle the hyperparameter optimization (HPO) problem.

%%% Narrower research area and statement of its importance: Multi-objective HPO on HPC
%%% Multi-objective
Even though the field of HPO has made significant advances in the context of single-objective optimization it seems now insufficient to consider only one objective as the set of competing objectives (metrics) to improve has grown. In recent years, it has become necessary for many applications to consider additional metrics, such as privacy~\cite{dwork2008differential}, bias and social fairness~\cite{mehrabi2021survey}, calibration~\cite{song2021classifier}, predictive uncertainty~\cite{begoli2019need,gawlikowski2023survey}, explainability~\cite{burkart2021survey,tjoa2020survey}, adversarial robustness~\cite{muhammad2022survey}, temporal and memory complexity~\cite{tan2019mnasnet}.
%%% HPC
Equally important, recent progress in ML was possible thanks to an increase in parallel computation~\cite{jordan2015machine}. As the multi-objective problem is significantly more challenging than single-objective, a promising approach to improve consists in scaling the search to leverage multiple compute units (such as GPUs) which we study in this work.

%%% Identification of a gap or other need for research
After reviewing the current multi-objective hyperparameter optimization (MOHPO) literature we identify a few gaps. First, we notice that the strongest contenders such as NSGA-II, a variant of genetic algorithms, lack iteration efficiency (i.e., improvement per completed black-box evaluations). This is important for expensive black-box such as in HPO. Second, we notice that most methods do not take into account practical considerations such as exploring only interesting trade-offs. For example, accuracy can be so low that such a model (even if faster) would not be a good candidate for other metrics. Third, we notice that model-based methods often do not provide solution diversity and require specific  adaptation and customization (e.g., choice of scalarization function, trade-off weights, normalization of objectives, tuning of a surrogate). Finally, we notice that some model-based methods do not benefit from scaling to more parallel resources as they do not gain significantly in performance or can simply not keep up with the demand (e.g., Gaussian-process-based methods).

%%% Specific research question meeting the identified need
We focus on developing a method for MOHPO that addresses the presented limitations: (a) improving iteration efficiency, (b) only exploring objectives trade-offs of interest, and (c) using parallelism to improve the effectiveness of the HPO.
%%% Summary of the approach to answer the research question
To that end, we develop a parallel search approach called decentralized multi-objective Bayesian optimization (D-MoBO) that combines a Random-Forest model, a randomized scalarization, a quantile transformation of objectives, and a penalty function. 
%%% Announcement of principal findings
The methodological contributions are:
\begin{enumerate}
    \item A quantile uniform normalization of objectives gives more importance to the approximated Pareto-Front in the hypervolume indicator and is also robust to outliers.
    \item A penalty to avoid ``out of interesting range'' objective trade-offs.
    \item A parallel search boosts optimization performance.
\end{enumerate}

\section{Background on Multi-Objective Optimization}
\label{sec:background}

%%% Generic MOO definition
In a generic multi-objective optimization (MOO) problem, the cost function is a vector-valued $f: \cX \rightarrow \cY$, $\cY = \mathbb{R}^\No$, where $\No$ is the number of objectives which are either maximized or minimized. Denote each component of $f$ by $f_i(x)$, and without loss of generality, assume that our goal is to minimize all components; then the MOO problem can be written as
\begin{align}
    \label{eq:generic-moo}
    & \min_{x \in \cX} (f_1(x), ..., f_\No(x)) \\
    & ~\text{s.t. } c(x) = 1 \nonumber
\end{align}
where $c: \cX \rightarrow \{0,1\}$ represents a constraint function.
Then the {\sl feasible domain} is the set of points in input space which respect the constraint, $\bar{\cX} := \{ x : c(x) = 1, \forall x \in \cX \}$.
The {\sl attainable set} is the image by $f$ of the {\sl feasible domain} $\bar{\cY} := \{ f(x) : \forall x \in \bar{\cX} \}$.
In most applications, it is not possible to simultaneously minimize all objectives as they often ``conflict'' with each other, which means that an improvement in one objective deteriorates other objectives (otherwise the MOO problem would be degenerated and therefore equivalent to a single objective problem).
Therefore, the {\bf solution to (\ref{eq:generic-moo}) is a {\sl set} of possibly infinite cardinal}.
Because these solutions may be incomparable, the solution set is defined via partial ordering as opposed to total ordering (as in the single-objective case).
To simplify the notation we will denote $f(x) := y$ and $f_i(x) := y_i$. 

%%% Some important math definitions that are needed to define the PF
\begin{definition}[Dominance partial ordering]
For two points in the {\sl attainable set}, $y^{(1)},y^{(2)} \in \bar{\cY}$,
$y^{(1)}$ is said to \emph{\bf dominate} $y^{(2)}$, written
$y^{(1)} \prec y^{(2)}$ if and only if for all
$i \in [1, \No] ~ y^{(1)}_i \leq y^{(2)}_i$  and, there exist $j \in [1,\No]$ such that $y^{(1)}_j < y^{(2)}_j$.
\end{definition}

With this definition, it is possible to define the notion of {\sl Pareto optimality}.

\begin{definition}[Non Dominance]
A point of the attainable set $y^\star = f(x^\star)$, $x^\star \in \bar{\cX}$ is said to be
\emph{\bf non-dominated} if and only if $y \not\prec y^\star$
for all $y \in \bar{\cY}$.
The corresponding $x^\star$ is said to be \emph{\bf efficient}.
\end{definition}

%%% Definition of nondominance and efficiency
Then the solution to (\ref{eq:generic-moo}) is called the {\sl Pareto-Frontier} (PF), which is the set of all non-dominated points formally defined as $\mathcal{F} := \{y^\star : y^\star = f(x^\star) \text{ is not dominated }\}$. The corresponding set of all efficient points $x^\star$ is called the {\sl Pareto-Set} (PS).

%%% PF is typically No - 1 dimensional
In general, when all objectives are conflicting and continuous, the PF is a $(\No-1)$-dimensional trade-off surface embedded in $\cY$ (e.g., with 2 objectives the PF is a curve).
To read more about MOO definitions and terminology, see \cite[Ch.~1-2]{ehrgott2005multicriteria}.

\subsection{Performance Indicators}
\label{sec:performance-indicators}

%%% How MOO is solved in practice
As in our set of methods of interest, it is not possible to return an infinite set of solutions to cover the PF, we must approximate it via a discrete set instead. Therefore a central challenge of MOO research is to measure the quality of the approximation of the PF.

%%% Challenges in evaluating MOO performance
To evaluate how well a discrete set of approximate solutions describe the PF's true shape is an open problem, and many metrics of MOO performance have been proposed \cite{audet2021performance}.
One desirable property of a MOO performance indicator is {\sl Pareto compliance}.
Let $A$ and $B$ be the sets of (approximately) non-dominated solutions returned by two different algorithms.
Then $A \prec B$ if for every $b \in B$, there exists $a \in A$ such that $a \prec b$.
An indicator $I$ is {\sl Pareto compliant} if either $A \prec B$ implies that $I(A) < I(B)$ or $I(A) > I(B)$ (depending on whether the indicator is increasing or decreasing with improved quality).
To our knowledge, the only MOO performance indicators that possess this property are the {\bf hypervolume indicator} (HVI) and the {\bf improved inverse generational distance} ($\IGD$) \cite{ishibuchi2015modified}.
The {\bf improved generational distance} ($\GD$) will also be of interest to us, although it does not possess this property.
All of these methods suffer from one drawback in that they rely on an appropriate choice of one or more reference points, which may require {\sl a priori} knowledge of the true PF.

%%% The HV indicator and its benefits and drawbacks
The HVI is given by $HVI(A) = V(\cup_{a\in A}[a, r])$, where $V(\cdot)$ denotes the volume, $[a, r]$ denotes a hypercube with lower bound $a$ and upper bound $r$, and $r$ denotes the reference point, which must be dominated by every solution point (a.k.a., the \emph{Nadir point}).
As mentioned above, the HVI is Pareto compliant.
In practice, it is typically possible to select the reference point for the HVI by using unacceptably bad scores for each objective.
However, an overly poor choice of reference points can lead to non-interpretable large values.
The $HVI$ is also extremely sensitive to poor problem scaling.
Finally, it is worth noting that $HVI$ is exponentially expensive to calculate when $\No>2$ \cite{ishibuchi2015modified}.

%%% IGD+ and GD+
The $\IGD$ and $\GD$ indicators are defined in terms of a modified distance metric $d^{+}(\hat{y}, y) := \|\max(\hat{y} - y, 0)\|_2$, where $y$ is a target point, $\hat{y}$ is an objective point in the estimated PF, and the $\max$ is taken element-wise. Then given a set of target points $Y$, the $\GD$ indicator is given by $\GD(\hat{Y}; Y) := \sum_{\hat{y}\in \hat{Y}}\min_{y\in Y}d^{+}(\hat{y}, y) / |\hat{Y}|$, and the $\IGD$ indicator is given by $\IGD(\hat{Y}; Y) := \sum_{y\in Y}\min_{\hat{y}\in \hat{Y}}d^+(\hat{y}, y) / |Y|$ (with $|.|$ the cardinal operator). In practice, these indicators are faster to compute than HVI. However, the need for a large target set $Y$ that ``covers'' the entire PF can be impossible to satisfy when the true PF is unknown, making these indicators difficult to use in real-world applications.

%%% Quality vs diversity and using HVI together with GD+
Balancing solution {\sl quality} (i.e., closeness between approximated and true PF) against {\sl diversity} (i.e., coverage of the approximated PF) is considered one of the central challenges of MOO \cite{audet2021performance,deb2002fast}.
While HVI and $\IGD$ are both Pareto compliant, they are redundant in that both tend to place a higher emphasis on diversity over quality. In this paper, we prefer using the HVI, which is more standard in the MOO literature.
But, for algebraic problems available in our appendix where the true PF is known, we will also utilize $\GD$ as an orthogonal indicator, which places a higher emphasis on solution quality.

\subsection{Scalarization}
\label{sec:background-scalarization}

%%% This paragraph is about basic scalarization idea and definition
One classical approach to solving a multi-objective problem is to transform them into a single-objective problem by using a scalarization function $s_w : \cY \rightarrow \mathbb{R}$ \cite[Ch.~3,4]{ehrgott2005multicriteria}, parameterized by a weight vector $w\geq 0$ normalized to 1, which reflects the trade-off between objectives. The optimization problem now becomes

\begin{align}
\label{eq:scalarization}
    & \min_{x\in \cX} s_w(f_1(x),\ldots,f_\No(x)) \\
    & ~\text{s.t. } c(x) = 1 \nonumber
\end{align}

%%% Some more details on scalarization
For most common scalarizations $s_w$, each choice of $w$ produces a different solution $x^\star$ to (\ref{eq:scalarization}), such that $x^\star$ is efficient.
By solving many instances of Problem~\ref{eq:scalarization} with different $w$, numerous solutions are produced giving an approximation to the PS. Many existing scalarization functions were proposed in the literature~\cite{chugh2020scalarizing}, and one way to achieve parallelism is by solving numerous independent scalarizations~\cite{chang2023parmoo,deb2013evolutionary}.

%%% Well-known drawbacks of scalarization
However, scalarization can be sensitive to scales and curvatures of objectives and often fails to produce complete coverage of the approximated PF. In particular, such characteristics can cause different trade-off parameters $w$ to produce similar or identical solution points $x^\star$. This can result in grouped solutions on the approximated PF and may leave certain regions sparsely populated or entirely empty. Such unbalanced and clustered coverage can give the user a biased understanding of the objectives trade-off.

%%% How scalarization is done in MOO literature
To resolve such limitations it is possible to adaptively select weights such that the approximated PF has a better coverage \cite{das1998,deshpande2016multiobjective}. However, this can induce a sequential dependence between different instances of Problem~\ref{eq:scalarization}, which can limit parallelism.

%%% Common alternatives to scalarization
Alternatives to scalarization also exist and are often based on maximizing the HVI. In Bayesian optimization (BO), one would often maximize the expected improvement in the HVI in each iteration.
The direct drawback of such an approach is the computational expense of calculating expected HVI improvement with more than two objectives, so one typically needs to approximate instead~\cite{daulton2020differentiable}.
Another kind, more scalable, is based upon multi-objective generalizations of genetic algorithms, which sort points according to the non-dominance relation during the selection phase. The most popular in this class is the non-dominated sorting genetic algorithm (NSGAII) \cite{deb2002fast}, which we use as a baseline in our experiments. One downside is that genetic algorithms can require several generations before they become significantly different from random sampling, and many more to converge. An efficient parallelization through the island model also exists for NSGAII \cite{martens2013asynchronous}.

\subsection{Hyperparameter Optimization}
\label{sec:hyperparameter-optimization}

%%% Multi-Objective Optimization in HPO
In HPO, $\cX$ is often a mixed-integer search space composed of categorical, discrete, and continuous variables. The target objectives are non-smooth. The optimized workflow is expensive to evaluate (in memory, in time). The run-time of evaluated workflows depends on the hyperparameters and has variability. Also, failures (a.k.a., hidden constraints) can appear for some hyperparameter configurations (e.g., out-of-memory, neural network training resulting in \texttt{NaN}).

%%% HPO Softwares
To our knowledge, very few open-source software have features that can handle all these requirements and also perform asynchronous large-scale parallelism.
Some notable software include pymoo~\cite{blank2020pymoo} (official implementation of NSGAII and other genetic algorithms), BoTorch~\cite{balandat2020botorch} (focus on Bayesian optimization), ParMOO~\cite{chang2023parmoo} (customized MOO algorithms), DeepHyper~\cite{deephyper_software} (parallel AutoML for HPC), Optuna~\cite{akiba2019optuna} (HPO solvers).
More details about some software features can be found in \cite[Table~1]{chang2023designing} and \cite{karl2022multiobjective}. In this study, we will focus on DeepHyper for its asynchronous BO solver and Optuna for NSGAII and MoTPE~ \cite{ozaki2022multiobjective,ozaki2020multiobjective} solvers which offer the most relevant features. We do not consider BoTorch as it is not suitable for large evaluations budgets \cite{optunadocs}.

\section{Method}
\label{sec:method}

\begin{figure*} %[!t]
    \centering
    \begin{subfigure}{0.33\textwidth}
        \centering
        \includegraphics[width=\textwidth]{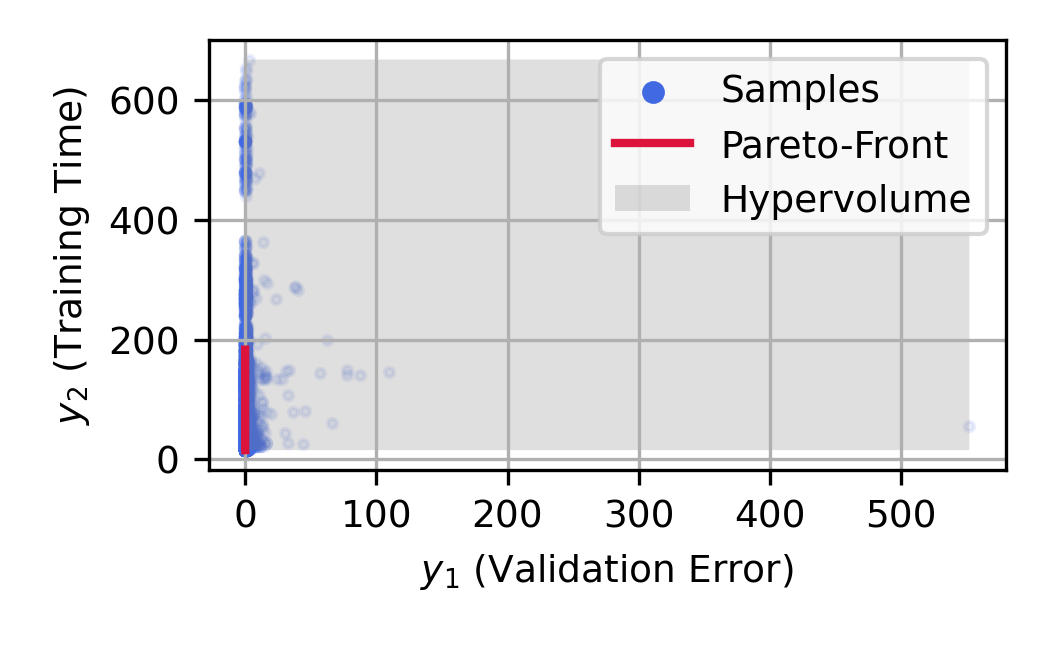}
        \caption{Identity (Id)}
        \label{fig:scaler_identity_navalpropulsion}
    \end{subfigure}
    \begin{subfigure}{0.33\textwidth}
        \centering
        \includegraphics[width=\textwidth]{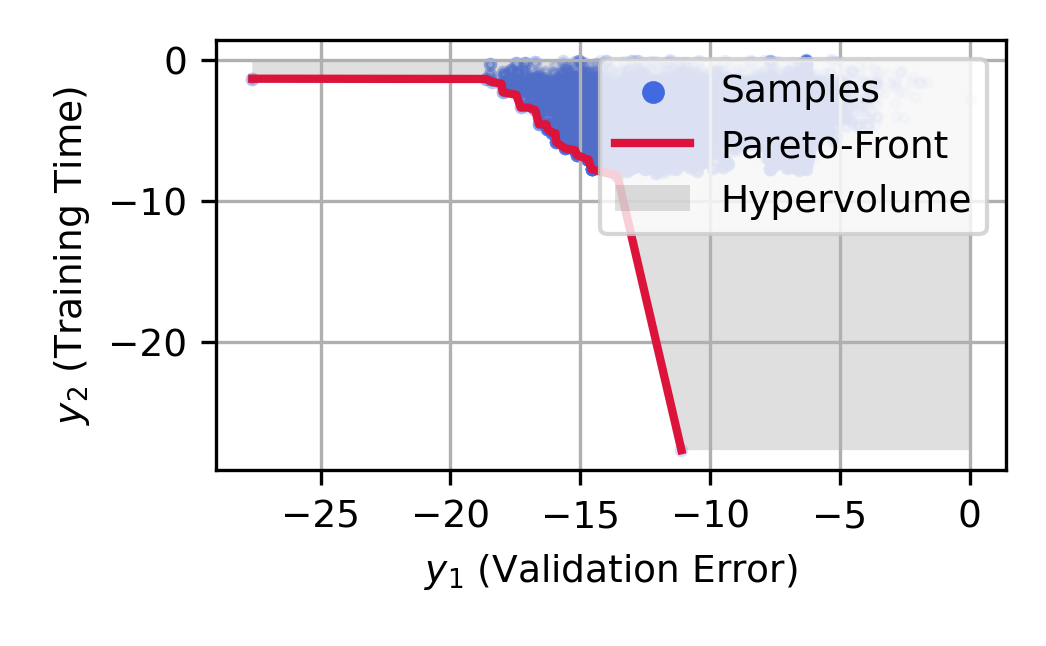}
        \caption{MinMax-Log (MML)}
        \label{fig:scaler_minmaxlog_navalpropulsion}
    \end{subfigure}
    \begin{subfigure}{0.33\textwidth}
        \centering
        \includegraphics[width=\textwidth]{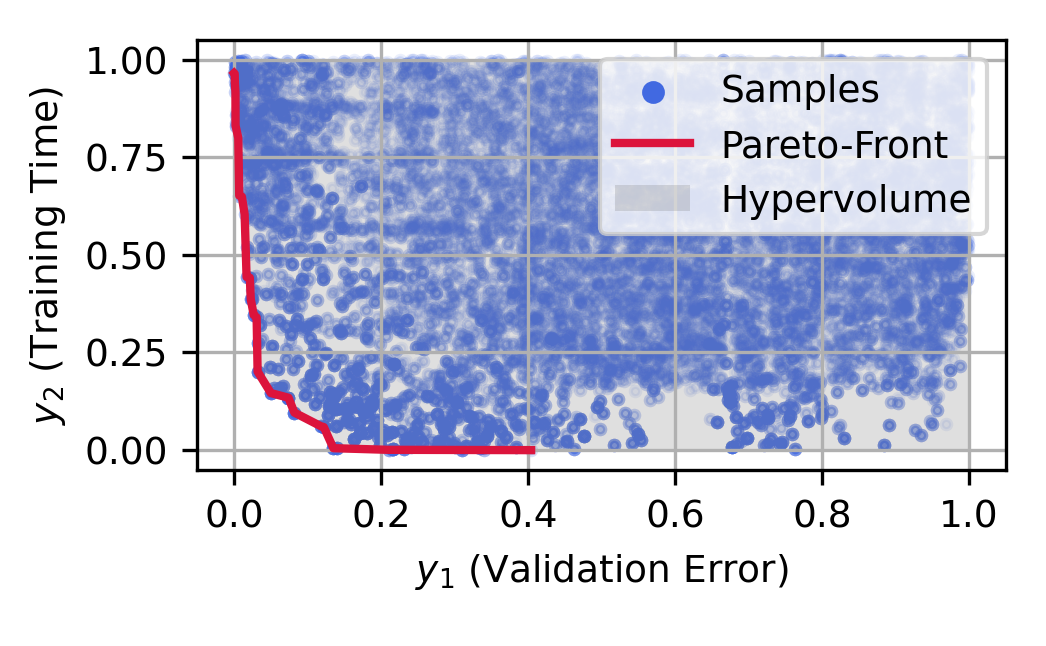}
        \caption{Quantile-Uniform (QU)}
        \label{fig:scaler_quantile-uniform_navalpropulsion}
    \end{subfigure}
    \caption{Comparing normalization of objectives on a 2-objectives hyperparameter optimization instance (NavalPropulsion from HPOBench) where both validation error $y_1$ and training time $y_2$ are minimized.}
    \label{fig:objective-scaling-methods}
\end{figure*}

In this section, we propose our decentralized multi-objective Bayesian optimization (D-MoBO) algorithm. It combines a parallel decentralized  asynchronous architecture with a sequential MoBO algorithm.

\subsection{Decentralized Bayesian Optimization}
\label{sec:method-decentralized-bo}

% Paragraph on Asynchronous Decentralized BO
The asynchronous decentralized scheme we use was proposed in~\cite{egele2022asynchronous}. The main idea is to start independent Bayesian optimization agents in parallel. Each agent is running a sequential BO algorithm but stores its observations in a shared memory space (e.g., database). The trick is to initialize different exploitation-exploration parameters combined with an exponential-decay scheduler on this parameter to avoid ``over''-exploring when increasing workers. The detailed algorithm is presented in the appendix.

\subsection{Multi-Objective Bayesian Optimization}
\label{sec:method-multi-objective-bo}

% Paragraph on Multi-Objective BO

The multi-objective Bayesian optimization (MoBO) algorithm we propose is presented in Algorithm~\ref{alg:mobo-algorithm}. It is inspired by the ParEGO~\cite{parego2006} algorithm which performs scalarization through the augmented Chebyshev function and the SMAC algorithm~\cite{smac2012} which uses a Random-Forest (RF) surrogate model with random-splits. A similar variant is already available in the SMAC3~\cite{smac3_2022} package.

%%% Quantile Uniform Normalization
\subsubsection{Normalization}
\label{sec:method-normalization}

{\it (Lines 8-9, Algorithm~\ref{alg:mobo-algorithm})} In the HPO setting, objectives of interest can have \emph{different scales} (e.g., accuracy, latency, FLOPS). Not only that, \emph{outliers} are also common when exploring a large hyperparameter search space (e.g., diverging metrics, numerical errors), see Figure~\ref{fig:scaler_identity_navalpropulsion} which display typical observations from MOO on the NavalPropulsion task from HPOBench~\cite{klein_tabular_2019,eggensperger_hpobench_2021} (other tasks of the considered benchmarks displayed similar behavior). 
On the one hand, the combination of both effects can make the HVI computation highly non-trivial. For example, the objective with the largest scale can weigh excessively and hide improvements in other objectives. Also, choosing a normalization and a reference point can become dependent on the experiment due to sensitivity to outliers such as in Figure~\ref{fig:scaler_identity_navalpropulsion} where most of the HVI between the top-right corner of the figure and the PF (red line) is empty.
On the other hand, the RF model used for BO can typically ``under-fit'' optimal configurations which often have the smallest squared error (e.g., closer to zero) already in the single-objective optimization setting. Previous work on single-objective optimization~\cite{egele_asynchronous_2022,smac3_2022}, usually apply some sort of ``log''-based transformation to mitigate the under-fitting problem. For instance, the MinMax-Log transformation $t^\text{MML}(y) = \log\left(({y - y_\text{min}})/{y_\text{max}} + \epsilon\right)$
is usually effective. However, we noticed that such transformation is also sensitive to outliers in the multi-objective case and while it was not really a problem for single-objective, it can transform a convex PF into a non-convex which makes the MOO problem much harder such as in Figure~\ref{fig:scaler_minmaxlog_navalpropulsion}.

% Why is scalarization usually bad
For these reasons, as discussed in background section, scalarization can perform poorly on real-world problems. In particular, it is well-known that uniformly sampled weights do \emph{not} produce uniformly distributed solution points on the PF, especially for linear scalarization~\cite{das1998}. 
To resolve these problems, we require a transformation that would conserve the PF properties, focuses on areas of interest (i.e., close to the estimated PF), and is robust to outliers. A mapping of the independent objective distributions $P(y_i)$ to the uniform distribution can provide such properties. This also helps apply randomized scalarization without worrying about differences in objective scales and curvatures.

To do so, we compose the empirical cumulative distribution function (ECDF) $\hat{F} : \mathbb{R} \rightarrow [0,1]$ with the quantile function (i.e., inverse of CDF) of the uniform distribution $Q_{\mathcal{U}(0,1)} : [0,1] \rightarrow \mathbb{R}$ which actually is the identity function $Q_{\mathcal{U}(0,1)}(x) = x$. This allows us to have a better update of the surrogate model. The ECDF is estimated from the observed objectives $\texttt{Y}$. The quantile-uniform (QU) transformation is $t^\text{QU}(y) = \hat{F}(y)$. The result of this transformation is illustrated in Figure~\ref{fig:scaler_quantile-uniform_navalpropulsion}.
%%% Introducing a formal proof that the above preserves Pareto optimality
To show that the above transformation preserves Pareto optimality, i.e., $y$ is non-dominated in ${\bar{\cal Y}}$ if and only if $t^{QU}(y)$ is non-dominated in $t^{QU}({\bar{\cal Y}})$, consider the following claim.

\begin{claim}[Invariance under order-preserving transformations]
If $t : \mathbb{R}^{N_o} \rightarrow \mathbb{R}^{N_o}$
is a componentwise order-preserving transformation on $y \in {\bar{\cal Y}}$
such that
$t_i(y^{(j)}) < t_i(y^{(k)}) \Rightarrow t_i(y^{(j)}) < t_i(y^{(k)})$
and vice versa, then
$t(y^\star)$ is nondominated in $t({\bar{\cal Y}})$ if and only if
$y^\star$ is nondominated in ${\bar{\cal Y}}$.
\end{claim}

\begin{proof}
Suppose that $t$ is order-preserving, as defined above.

($\Rightarrow$):
Assume that $t(y^\star)$ is Pareto optimal in $t({\bar{\cal Y}})$,
which implies that
$y^{(t,j)} \not\prec t(y^\star)$
for all $y^{(t,j)} \in t({\bar{\cal Y}})$.
For contradiction, suppose that $y^{(j)} \prec y^\star$ for some
$y^{(j)} \in {\bar{\cal Y}}$.
Then by definition $y^{(j)} \leq y^\star$ and $y^{(j)}_i < y^\star_k$
for at least one $i \in \{1, \ldots, o\}$.
Since $t$ is order-preserving, this would imply that
$t_i(y^{(j)}) < t_i(y^\star)$
and componentwise $t(y^{(j)}) \leq t(y^\star)$.
Since $t(y^{(j)}) \in t({\bar{\cal Y}})$, this is a contradiction.

($\Leftarrow$):
Assume that $y^\star$ is Pareto optimal in ${\bar{\cal Y}}$,
which implies that $y^{(j)} \not\prec y^\star$
for all $y^{(j)} \in {\bar{\cal Y}})$.
Each $y^{(t,j)}\in t({\bar{\cal X}})$ satisfies
$y^{(t,j)} = t(y^{(j)})$ for some point in the untransformed
set ${\bar{\cal X}}$.
But from the assumption, $y^{(j)} \not\prec y^\star$.
So by similar logic as above, $t(y^{(j)}) \not\prec t(y^\star)$.
\end{proof}

\begin{corollary}[Quantile Uniform Transformation]
The quantile transformation $t^{QU}$ is an ECDF.
By definition, an ECDF is monotone increasing (strictly monotone when invertible),
so it is immediately order-preserving.
So from the claim, we immediately conclude that $t^{QU}$ preserves the
Pareto set.
\end{corollary}

\begin{figure}[!b]
    \centering
    \begin{subfigure}{0.49\columnwidth}
        \centering
        \includegraphics[width=\textwidth]{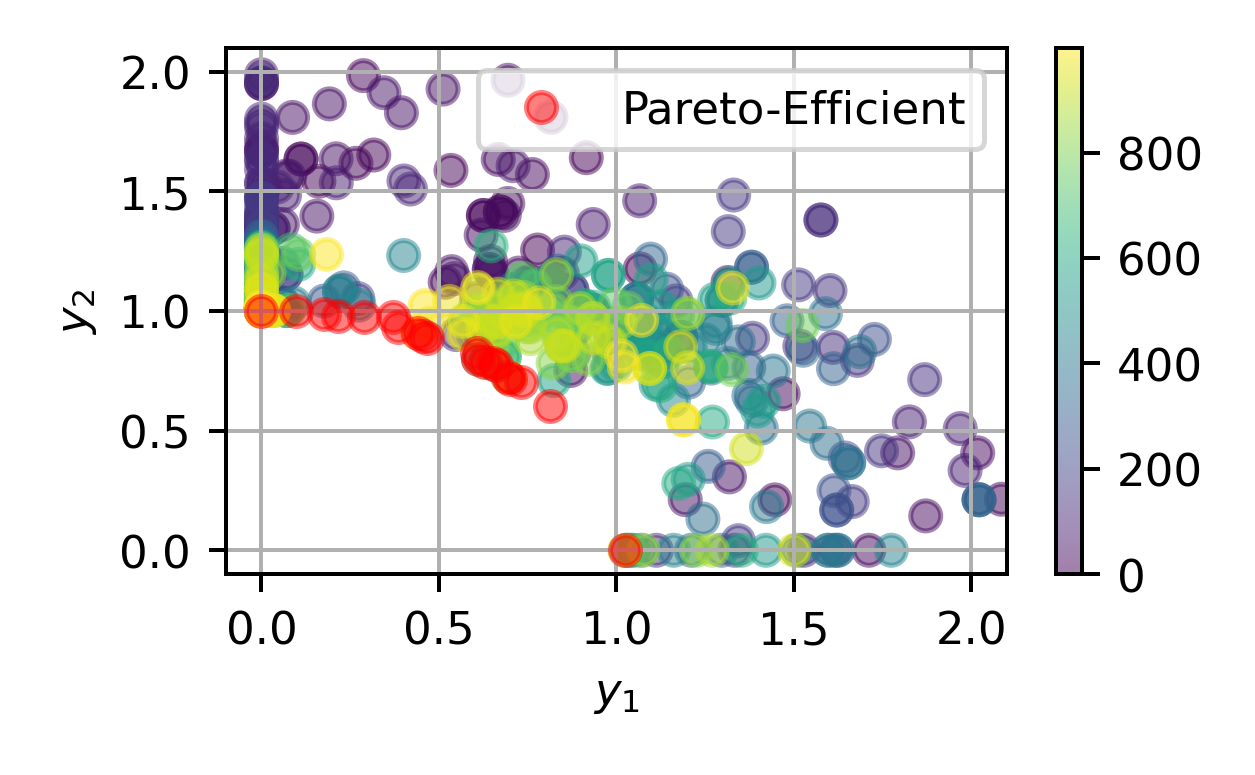}
        \caption{No Penalty (NP)}
        \label{fig:example-dtlz2-penalty-a}
    \end{subfigure}
    \begin{subfigure}{0.49\columnwidth}
        \centering
        \includegraphics[width=\textwidth]{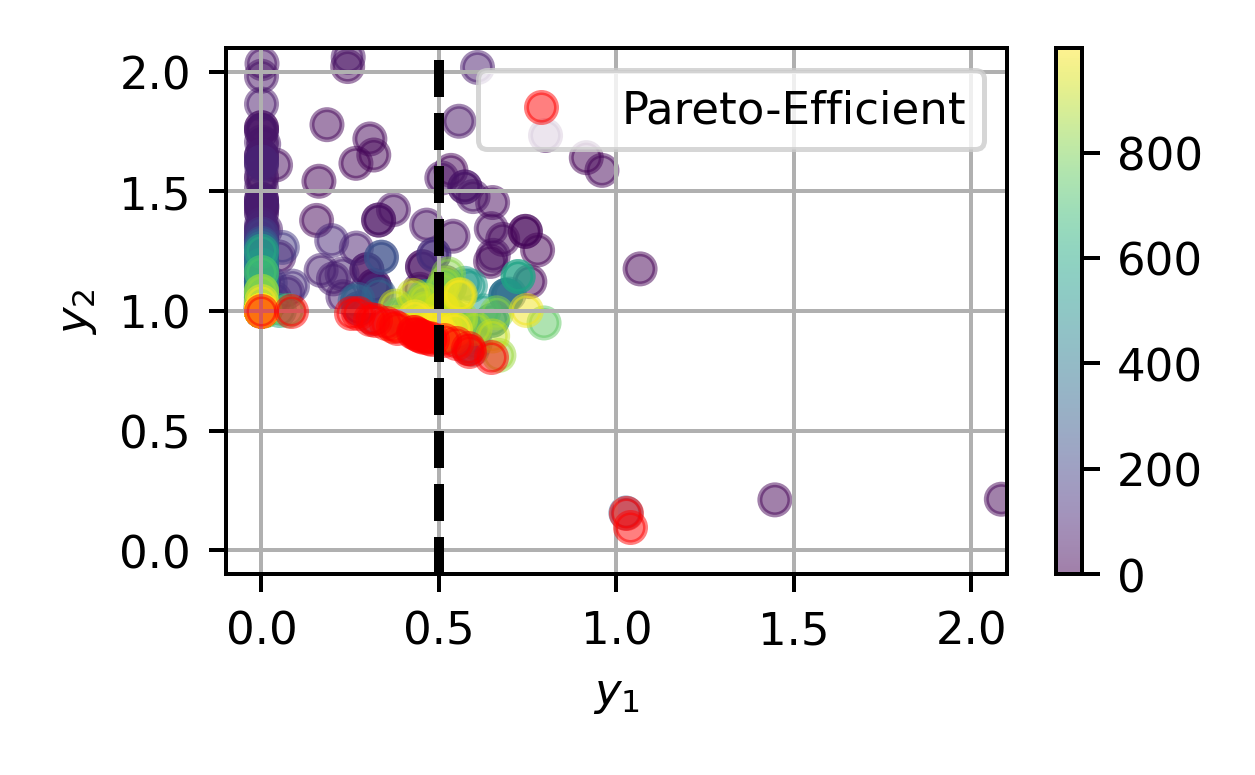}
        \caption{Penalty $y_1 < 0.5$}
        \label{fig:example-dtlz2-penalty-b}
    \end{subfigure}
    \caption{Example effect of the penalty on the DTLZ 2 benchmark~\cite{deb2005scalable}. The color of points indicates the number of evaluations of the function.}
    \label{fig:example-dtlz2-penalty}
\end{figure}

%%% Penalty Function motivation
\subsubsection{Penalty}

{\it(Lines 10-12, Algorithm~\ref{alg:mobo-algorithm})} Even though we are interested in exploring a diverse set of solutions on the PF, there are often minimal requirements on some objectives.
For example, if we have a baseline binary classifier with an error rate of 15\%, likely, we would not consider any solution configuration with an error rate greater than 20\%.
In the generic multi-objective Problem~\ref{eq:generic-moo}, there is no way to specify that some solutions are fundamentally less interesting than others. Therefore, we propose to impose upper bounds on the objective ranges. Since these constraints could be violated (i.e., obtaining some results out of bounds), we can consider these objective bounds as {\sl soft} constraints \cite{ledigabel2015}.

%%% Explaining where these techniques come from
In the single-objective literature, nonlinear constraints are often handled via a {\sl penalty function}, such as an augmented Lagrangian.
This technique generalizes to the multi-objective case, where the penalty must be applied to all objectives (not only the objective that violates its upper bound) \cite{cocchiandlapucci2020}.
In our case where the constraint functions are also black-box functions, the {\sl progressive barrier} penalty function has been shown to work well in the single-objective case \cite{audet2009}.
In other multi-objective black-box software, the progressive boundary approach has been successfully implemented in the multi-objective case and shown to be effective in handling arbitrary black-box constraints \cite{chang2023designing}.

%%% How we do it in this paper
To enforce upper bounds on objective ranges, we apply a progressive barrier penalty to all objectives whenever one or more objectives violate their upper bounds.
The penalty is calculated as the sum of all constraint violations multiplied by a penalty strength factor $\gamma$.
See the exact calculation in Algorithm~\ref{alg:mobo-algorithm}.
Note that, thanks to the QU transformation, it is appropriate to choose a problem-independent penalty constant of $\gamma=2$, which is slightly greater than the normalized objective magnitudes.
This penalty discourages the optimizer from wasting resources further refining uninteresting trade-offs, which fail to meet the minimal requirements. An example of the effect of such a penalty is provided in Figure~\ref{fig:example-dtlz2-penalty}. In the case where no penalty is applied (Figure~\ref{fig:example-dtlz2-penalty-a}) the full PF of the problem is explored until the end (yellow points). But, when applying the penalty to enforce $y_1 < 0.5$ (Figure~\ref{fig:example-dtlz2-penalty-b}) we observe that most of the PF where $y_1 > 0.5$ is rarely explored.

%% Important points about the penalty
We stress two limitations of such a penalty. First, when the true PF is not known using a penalty can be ineffective. Indeed, if the penalty does not impact the PF then we observed that the performance of the MOO algorithms is similar or worse than not having the penalty. Second, this penalty strategy is particularly effective for QU normalization. However, when using the same scheme with other transformations (Identity and MinMax-Log) it was very much ineffective or required harder tuning of the $\gamma$ parameter.

%%% Scalarization
\subsubsection{Scalarization}
\label{sec:method-scalarization}

{\it (Lines 13-14, Algorithm~\ref{alg:mobo-algorithm})} We focus on scalarization-based MoBO and as it is not clear which function is better we consider several functions from the literature \cite{chugh2020scalarizing}: weighted-sum or \emph{linear} (L), the \emph{Chebyshev} (CH), and the \emph{penalty-boundary intersection} (PBI). The detailed form of these functions is provided in the appendix. Then, to enhance the diversity of the estimated PF, in each BO iteration we decide to re-sample weights uniformly from the unit-simplex $\Delta N_o$ \cite[Remark 1.3]{pinelis2019order}
$ w_i = {\log(1 - \tilde{w}_i)}/({\sum_{j=1}^{N_o} \log(1 - \tilde{w}_j)})$ with $\tilde{w}_i \sim \mathcal{U}(0,1)$.

% DTLZ 2
% 2 objectives
% 8 input variables
% offset of 0.5

%%% Exact MoBO algorithm
\begin{algorithm2e}[!t]
\small
\SetInd{0.25em}{0.5em}
\SetAlgoLined

\SetKwInOut{Input}{Inputs}
\SetKwInOut{Output}{Output}
\SetKwProg{Fn}{Function}{ is}{end}

\SetKwFunction{suggest}{suggest}
\SetKwFunction{observe}{observe}

\SetKwFor{For}{for}{do}{end}

\Input{$\texttt{Ni}$: number of initial configurations, $\texttt{No}$: number of objectives, $\texttt{UB}$: objectives upper-bounds.}

$\texttt{X}, \texttt{Y} \gets $ New empty arrays of configurations/objectives \;
$\texttt{m} \gets $ new Random-Forest model \;
$\gamma \gets $ Initialize penalty strength $=2$\;

~\\

%%%%%%%% Optimizer.observe
\Fn{\observe{$\texttt{X\_new},\texttt{Y\_new}$}}{
    {\color{orange}\tcc{Updates model with observations}}
    $\texttt{X} \gets $ Concatenate $\texttt{X}$ with $\texttt{X\_new}$ \;
    $\texttt{Y} \gets $ Concatenate $\texttt{Y}$ with $\texttt{Y\_new}$ \;
    {\color{blue}\tcc{Quantile-Uniform Normalization}}
    $\hat{F} \gets $ Estimate empirical CDF from $\texttt{Y}$\;
    $\texttt{Yu} \gets \hat{F}(\texttt{Y}) $ \;
    {\color{blue}\tcc{Apply Penalty}}
    $\texttt{UBu} \gets \hat{F}(\texttt{UB}) $ \;
    $\texttt{p} \gets \gamma \sum_{i \in [1;\texttt{No}]} \max(\texttt{Yu}[:,i] - \texttt{UBu}[i], 0) $ \;
    $\texttt{Yp} \gets  \texttt{Yu} + \texttt{p}$\;
    {\color{blue}\tcc{Scalarization and Update}}
    $\texttt{w} \gets $ Sample array of $\texttt{No}$ weights from $\Delta \texttt{No}$\;
    $\texttt{Ys} \gets $ Apply $s_\texttt{w}(.)$ on all elements of $\texttt{Yp}$ \;
    Replace failed evaluations in $\texttt{Ys}$ with $\max(\texttt{Ys})$ \;
    Update model $\texttt{m}$ with $(\texttt{X},\texttt{Ys})$ observations \;
}

~\\

%%%%%%%% Optimizer.suggest
\Fn{\suggest{$\kappa$}}{
    {\color{orange}\tcc{Minimizes lower-confidence bound}}
    \eIf{Length of \texttt{X} $ < \texttt{Ni}$}{
        $\texttt{x} \gets$ Sample random configuration \;
    }{
        $\texttt{x} \gets \argmin_x \texttt{m}_\mu(x) - \kappa \cdot \texttt{m}_\sigma(x)$ \;
    }
    \Return $\texttt{x}$ \;
}

\caption{Multi-Objective Bayesian Optimizer}
\label{alg:mobo-algorithm}
\end{algorithm2e}

\section{Results}
\label{sec:results}

In this section we present our experimental results on D-MoBO. First, we evaluate which combination of scalarization function and normalization is the most effective. Then, we evaluate the impact of the penalty function. Finally, we evaluate the gain in performance of D-MoBO when scaling parallel workers.

\subsection{Evaluating Scalarization and Normalization}
\label{sec:results-hpobench}

\begin{figure}[!h]
    \centering
    \includegraphics[width=\columnwidth]{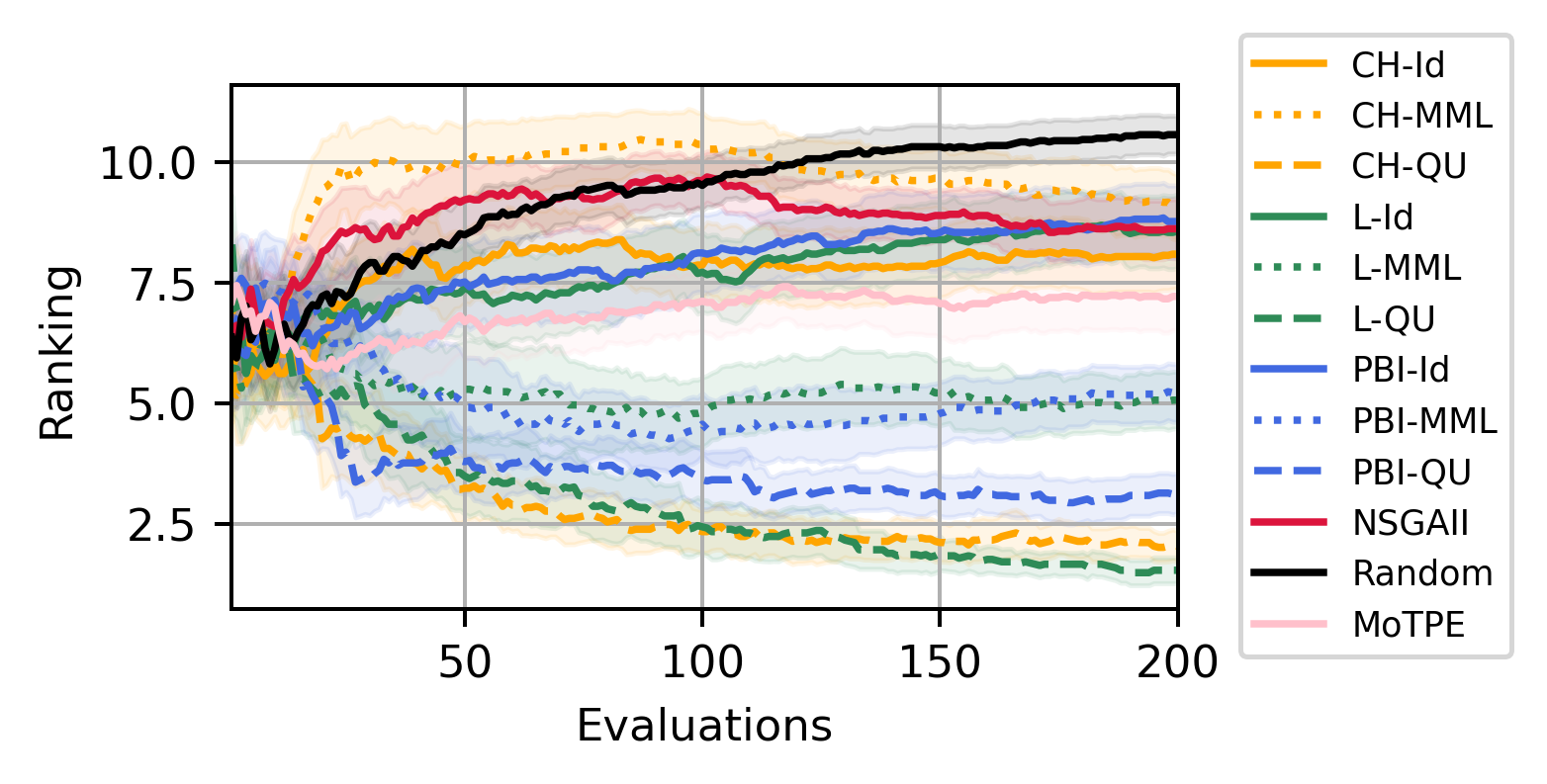}
    \caption{Ranking of scalarization functions and objective normalization combinations on the HPOBench tabular tasks (i.e., 40 experiments per curve).}
    \label{fig:hpobench_average_ranking}
\end{figure}

% Results on HPOBench: selecting the scalarization function and objective scaling method
In this section, we present the experimental results which led us to choose the Linear scalarization with QU normalization. For this, we select four HPO benchmarks tasks from HPOBench~\cite{eggensperger_hpobench_2021,klein_tabular_2019}: NavalPropulsion, ParkinsonsTelemonitoring, ProteinStructure, SliceLocalization. The 2 objectives are the validation error and the training time.
These benchmarks allow for fast evaluations of HPO methods as they simulate learning workflow evaluations through a pre-computed database. 
Also, in this experiment, we limit ourselves to studying the behavior of Algorithm~\ref{alg:mobo-algorithm} for 200 sequential iterations. 
For the comparison, we choose 3 scalarization functions: Linear (L), Chebyshev (CH), and PBI. And, 3 objective normalization strategies: Identity (Id), MinMax-Log (MML), and Quantile-Uniform (QU). As we want to study the synergy between scalarization and objective normalization each combination (9 in total) is evaluated. We add two baselines for the comparison: random (Random), NSGAII, and MoTPE.
Each configuration is repeated 10 times with fresh random states. The quality metric is HVI.
In Figure~\ref{fig:hpobench_average_ranking}, we present the averaged ranks over HVI curves across all tasks and repetitions (i.e., 40 experiments). The average HVI curves are given for transparency in the appendix. Scalarization functions have different colors and objective normalization have different line styles. 
The transparent bandwidth around the curves represents a confidence interval of 95\% confidence of the mean-value estimation s (i.e., 1.96 standard error).

From the ranking curves (Figure~\ref{fig:hpobench_average_ranking}) (lower ranks are better), it is clear that {\bf QU normalization is the best} as strategies based on QU (dashed lines) have curves which dominate all other methods independently of the scalarization function.
Then for the scalarization functions, it appears that PBI (blue dashed line) is underperforming compared to L (green dashed line) and CH (orange dashed line). As L is winning we keep it as the default scalarization function. However, because L and CH are statistically different (i.e., overlapping standard errors) and also because we know that CH would be more robust in the case of non-convex PF we keep it as an option for future work. Indeed, CH may be harder to optimize for MoBO and could require a few iterations with fixed weights instead of always re-sampling weight vectors.
Finally, from 100 evaluations it seems like NSGAII manages to slowly improve its ranking.

\subsection{Exploring the Effect of Penalty}
\label{sec:results-combo-penalty}

\begin{figure}[!b]
    \centering
    \includegraphics[width=0.9\columnwidth]{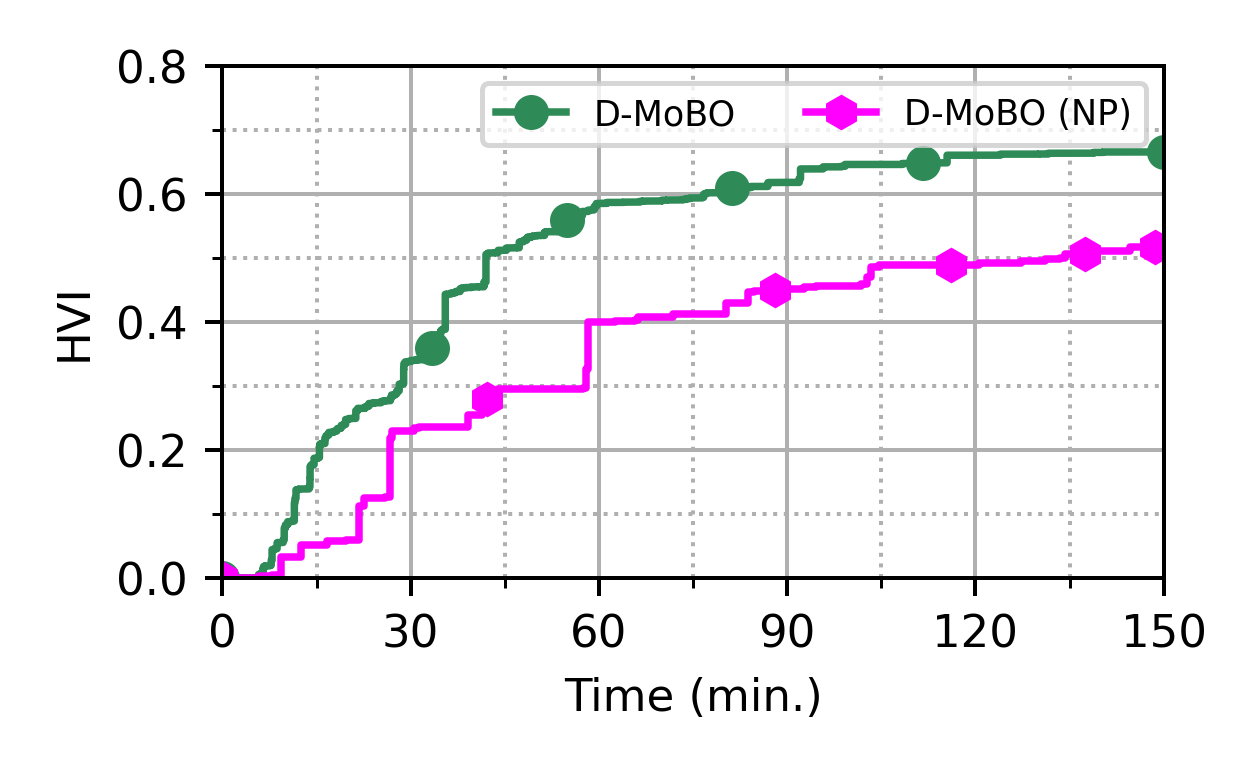}
    \caption{Observing the effect of penalty for D-MoBO on the Combo benchmark. The penalty is set to $R^2 > 0.85$. D-MoBO (green), is the version implementing the penalty while D-MoBO (NP) (pink) is without penalty. HVI is computed using the reference for $R^2=0.85$.}
    \label{fig:hypervolume-vs-time-polaris-combo-d-mobo-constraint-160.png}
\end{figure}

%%% Penalty helps avoid exploring un-interesting models with extremely low predictive power
In this section, we show the effect of the penalty to explore only ``useful'' hyperparameter configurations. For this, we use a different benchmark more suitable for the HPC setting. We choose the Combo problem from the ECP-Candle benchmark which was introduced in previous works for AutoML on HPC~\cite{balaprakash_scalable_2019,egele_agebo-tabular_2020,egele_asynchronous_2022}. This benchmark contains 22 hyperparameters and performs a regression task on the growth rate of Cancer cells given a treatment. The multi-objective problem is to minimize $y_1 := 1-R^2$ (i.e., maximizing $R^2$ the coefficient of determination on validation data), minimize $y_2$ the latency (i.e., inference time), and minimize the number of parameters (i.e., model size). Each model can train for a maximum of 50 epochs and 30 minutes. Experiments are run on a maximum of 640 GPUs in parallel (1 evaluation per GPU) for 2.5 hours (i.e., corresponds to only 5 sequential evaluations of 30 minutes). Only 1 repetition is done per experiment due to their cost and the random seed is fixed to 42 for all experiments. The penalty upper-bound is set to $y_1 < 0.15$ as the baseline model from Combo reaches $y_1=0.13$ after 100 epochs. 
In Figure~\ref{fig:hypervolume-vs-time-polaris-combo-d-mobo-constraint-160.png} we present the HVI vs time curve for D-MoBO (with penalty), and D-MoBO (NP) (without penalty). The reference point to compute the HVI is set $y_\text{ref} = (0.15, \max(y_2), \max(y_3))$. Such a penalty represents the practical consideration where any model with a predictive performance less than a threshold is unusable (minimum requirement). It is clear from the results that applying the penalty helps the algorithm focus on the solution set of interest as the HVI curve of D-MoBO increases much faster than D-MoBO (NP). Similar results with fewer parallel workers (40 and 160) are observed and presented in the appendix. Also, in Table~\ref{tab:partial-table} we provide additional quantitative metrics. Mainly, when comparing D-MoBO and D-MoBO (NP) we notice that without the penalty very few models are reaching the threshold of required accuracy. For example, the penalty helps, for the same computational budget, to move from 177 to 3,454 models with $y_1 < 0.15$.

%%% Comparison with other MOO algorithms
Then, we also compare D-MoBO with other noticeable MOO algorithms of the literature: NSGAII, MoTPE (both from Optuna with default parameters), and Random search. While Random search is mostly used as a sanity check of HPO algorithms in general, NSGAII is known to be a strong performer in the MOO field and MoTPE is known to have faster convergence. We used the vanilla version of these algorithms and therefore they do not enforce any bound on the objectives. In Figure~\ref{fig:hypervolume-vs-time-polaris-combo-all-160} we can observe that D-MoBO has the advantage in early iterations but NSGAII is bridging the gap in late iterations. This effect is confirmed with different numbers of workers, see the appendix. We can also see that MoTPE stagnates quickly to barely outperforms the Random search but it started slightly better than NSGAII. MoTPE was expected to not scale properly.

\begin{figure}[!h]
    \centering
    \includegraphics[width=0.9\columnwidth]{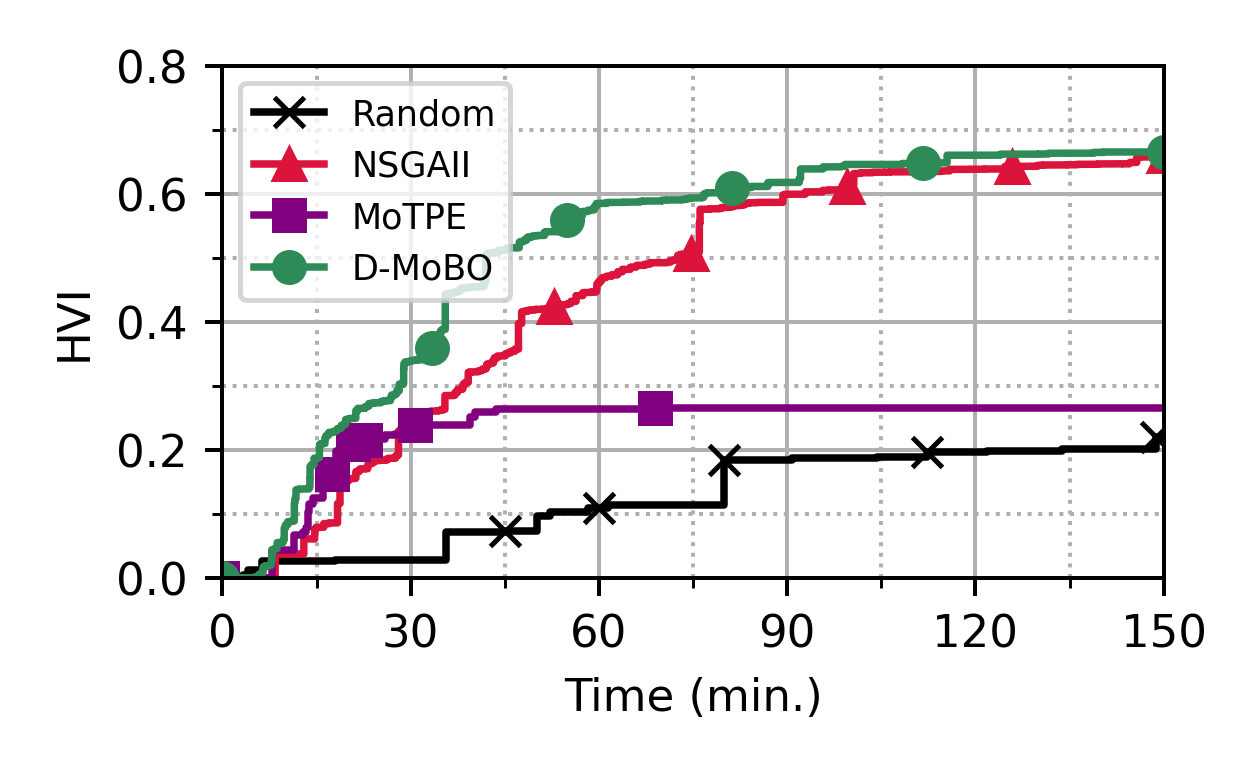}
    \caption{Comparing D-MoBO (with penalty) against other optimizers on the Combo benchmark.}
    \label{fig:hypervolume-vs-time-polaris-combo-all-160}
\end{figure}

\subsection{Boosting Performance with Parallelism}
\label{sec:results-combo-hpc}

% D-MoBO performance improves when increasing the number of workers
In this section, we show the gain in solution quality and convergence speed when increasing the number of parallel workers for D-MoBO. For this, we follow the same experimental setting as the previous section. In addition to the experiments with 640 GPUs, we also run D-MoBO (including the penalty) with 160 and 40 GPUs (i.e., varying the workers by a factor of 4). In Figure~\ref{fig:hypervolume-vs-time-polaris-combo-d-mobo-scaling.png}, we see that increasing workers improves both convergence speed and solution quality. For example, looking at the convergence speed, with 640 workers in 30 mins we reach the final solution of 40 workers in 2.5  (i.e., 5x speed-up). Then, for the solution quality, if we look at any point in time the experiments with more workers have larger HVI. To aggregate these, we provide, in Table~\ref{tab:partial-table}, the HVI indicator at the end of the experiment (HVI, a metric of solution quality) and the area under the curve (AUC, a metric of convergence speed) for each scale of workers. The AUC is commonly used in the ``any-time'' learning setting where the goal is to find the learning machines with faster convergence~\cite{autdl_challenge_2019}. 

% Other algorithm
Similar scaling experiments were performed for other MOO algorithms: NSGAII, MoTPE, and Random search. We notice that NSGAII (the strongest competitor in our study) gains similarly in performance when increasing the number of parallel workers. However, for MoTPE and Random search, the gain in performance is minor. Detailed results are provided in the appendix. Finally, we can see in Table~\ref{tab:partial-table} that D-MoBO consistently outperformed other competitors at different scales. However, the difference in the final solution against NSGAII becomes less significant when the computational budget increases.

\begin{figure}
    \centering
    \includegraphics[width=0.9\columnwidth]{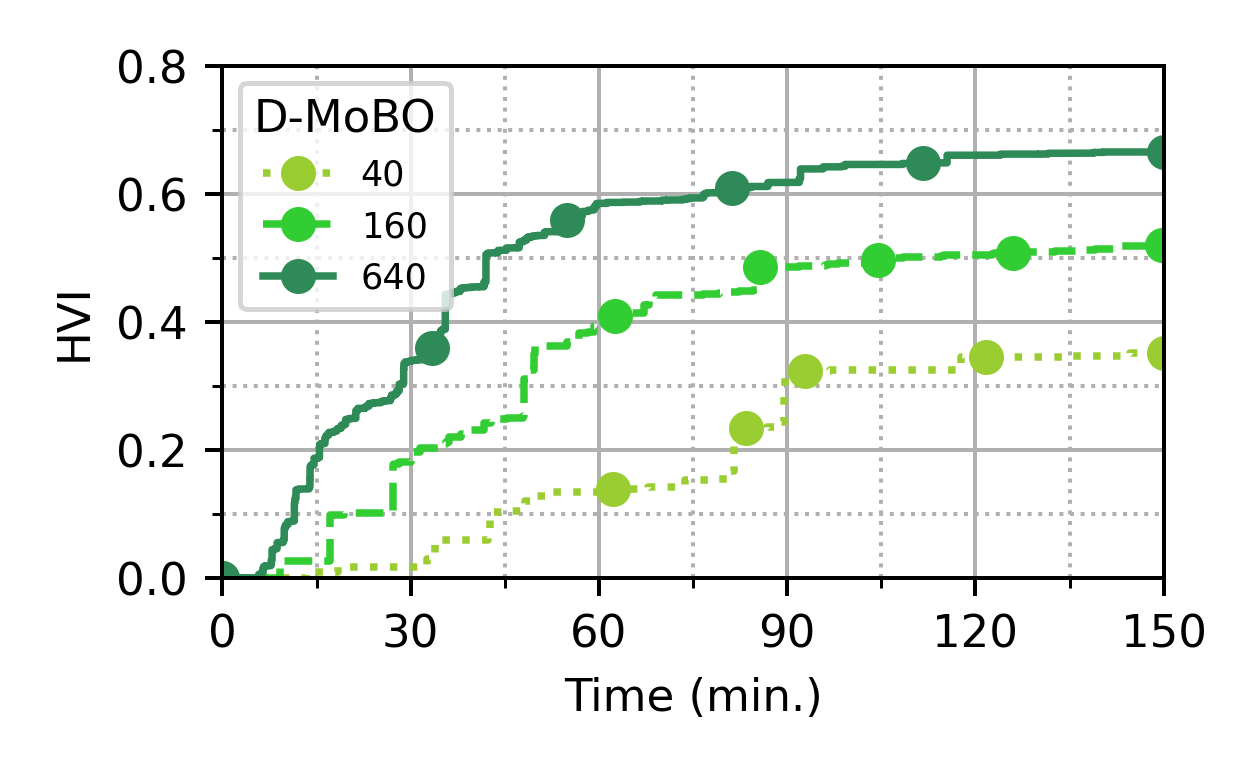}
    \caption{Observing the effect of increased parallel workers in D-MoBO. From 40, 160 to 640 parallel GPUs. The more workers the better is the solution and the convergence has 640  strongly dominates other settings.}
    \label{fig:hypervolume-vs-time-polaris-combo-d-mobo-scaling.png}
\end{figure}

\begin{table}
\centering
\resizebox{\columnwidth}{!}{%
\begin{tabular}{l|ccc|ccc|ccc}
    & \multicolumn{3}{c|}{\textbf{40 Workers}} & \multicolumn{3}{c|}{\textbf{160 Workers}} & \multicolumn{3}{c}{\textbf{640 Workers}} \\
    & \#VC   & HVI    & AUC      & \#VC  & HVI    & AUC    & \#VC   & HVI    & AUC   \\ \hline
\textbf{Random}     & 4      & 0.16  & 0.09  & 11    & 0.14  & 0.08  & 30     & 0.21  & 0.13 \\
\textbf{MoTPE}      & 1      & 0.08  & 0.07  & 21    & 0.14  & 0.11  & 125    & 0.26  & 0.23 \\
\textbf{NSGAII}     & 14     & 0.28  & 0.12  & 183   & 0.5  & 0.27  & 2,314  & \textbf{0.66}  & 0.45 \\
\textbf{D-MoBO} &
\textbf{109} &
\textbf{0.34} &
\textbf{0.19} &
\textbf{497} &
\text{0.51} &
\textbf{0.35} &
\textbf{3,454} &
\textbf{0.66} &
\textbf{0.51} \\
\textbf{D-MoBO NP} & 1      & 0.1  & 0.04  & 37    & 0.35  & 0.22  & 177    & 0.51  & 0.34
\end{tabular}%
}
\caption{Comparing optimizers for different numbers of workers. \#VC: the number of evaluations with valid objective bounds and AUC: area under the HVI curve. HVI and AUC are rounded to the closest value.}
\label{tab:partial-table}
\end{table}

\section{Discussion}

%%% Interpretation of results to answer research questions
In this study, it was shown that D-MoBO is a promising approach for MOHPO. Its performance is explained by (1) its QU normalization, (2) its soft penalty, and (3) its capacity to benefit from an increase in parallel workers.

%%% Consideration of possible weaknesses
However, we also observed that the difference between NSGAII and D-MoBO becomes less significant with an increased number of evaluations. We note that only NSGAII compatible with mixed-integers was tested but not without constraints which would make a better comparison. Lastly, the penalty can be hard to set without prior information about the problem (i.e., hard to automate).

%%% Relationship of results to previous literature and broader implications of having answered research question
Contrary to some previous benchmarks such as YAHPO-Gym~\cite[Figure 4]{pfisterer2022yahpo} we find that MOHPO based on BO can significantly outperforms other competitors. In addition, we find that random search does not gain significantly from an increase in parallel computations even if a lot more evaluations are completed.

%%% Prospects for future progress
In future works, the policy to sample trade-off weights should be optimized as uniform sampling is simple but under-optimized and certainly not adapted to all considered scalarization functions. 
Also, the scalarization function study should be refined by testing more MOHPO problems to evaluate how frequent are non-convex PF. For this, we wish to use JAHSBench-201~\cite{bansal2022jahs} and YAHPO-Gym which are two promising benchmarks for MOHPO currently being refined.

\section{Acknowledgment}
This material is based upon work supported by the U.S.\ Department of Energy 
(DOE), Office of Science, Office of Advanced Scientific Computing Research, under
Contract DE-AC02-06CH11357. This research used resources of the Argonne 
Leadership Computing Facility, which is a DOE Office of Science User Facility. 
This material is based upon work supported by ANR Chair of Artificial Intelligence HUMANIA ANR-19-CHIA-0022 and TAILOR EU Horizon 2020 grant 952215.

\bibliography{dh-moo}

\begin{thebibliography}{48}
\providecommand{\natexlab}[1]{#1}

\bibitem[{Akiba et~al.(2019)Akiba, Sano, Yanase, Ohta, and Koyama}]{akiba2019optuna}
Akiba, T.; Sano, S.; Yanase, T.; Ohta, T.; and Koyama, M. 2019.
\newblock Optuna: A next-generation hyperparameter optimization framework.
\newblock In \emph{Proceedings of the 25th ACM SIGKDD international conference on knowledge discovery \& data mining}, 2623--2631.

\bibitem[{Audet et~al.(2021)Audet, Bigeon, Cartier, Le~Digabel, and Salomon}]{audet2021performance}
Audet, C.; Bigeon, J.; Cartier, D.; Le~Digabel, S.; and Salomon, L. 2021.
\newblock Performance indicators in multiobjective optimization.
\newblock \emph{European journal of operational research}, 292(2): 397--422.

\bibitem[{Audet and Dennis(2009)}]{audet2009}
Audet, C.; and Dennis, J.~E. 2009.
\newblock A Progressive Barrier for Derivative-Free Nonlinear Programming.
\newblock \emph{SIAM Journal on Optimization}, 20(1): 445--472.

\bibitem[{Balandat et~al.(2020)Balandat, Karrer, Jiang, Daulton, Letham, Wilson, and Bakshy}]{balandat2020botorch}
Balandat, M.; Karrer, B.; Jiang, D.; Daulton, S.; Letham, B.; Wilson, A.~G.; and Bakshy, E. 2020.
\newblock BoTorch: A framework for efficient Monte-Carlo Bayesian optimization.
\newblock \emph{Advances in neural information processing systems}, 33: 21524--21538.

\bibitem[{Balaprakash et~al.(2019)Balaprakash, Egele, Salim, Wild, Vishwanath, Xia, Brettin, and Stevens}]{balaprakash_scalable_2019}
Balaprakash, P.; Egele, R.; Salim, M.; Wild, S.; Vishwanath, V.; Xia, F.; Brettin, T.; and Stevens, R. 2019.
\newblock Scalable {Reinforcement}-{Learning}-{Based} {Neural} {Architecture} {Search} for {Cancer} {Deep} {Learning} {Research}.
\newblock In \emph{Proceedings of the {International} {Conference} for {High} {Performance} {Computing}, {Networking}, {Storage} and {Analysis}}, 1--33.
\newblock ArXiv:1909.00311 [cs, stat].

\bibitem[{Bansal et~al.(2022)Bansal, Stoll, Janowski, Zela, and Hutter}]{bansal2022jahs}
Bansal, A.; Stoll, D.; Janowski, M.; Zela, A.; and Hutter, F. 2022.
\newblock {JAHS-Bench-201}: A Foundation For Research On Joint Architecture And Hyperparameter Search.
\newblock In \emph{Thirty-sixth Conference on Neural Information Processing Systems Datasets and Benchmarks Track}.

\bibitem[{Begoli, Bhattacharya, and Kusnezov(2019)}]{begoli2019need}
Begoli, E.; Bhattacharya, T.; and Kusnezov, D. 2019.
\newblock The need for uncertainty quantification in machine-assisted medical decision making.
\newblock \emph{Nature Machine Intelligence}, 1(1): 20--23.

\bibitem[{{Blank} and {Deb}(2020)}]{blank2020pymoo}
{Blank}, J.; and {Deb}, K. 2020.
\newblock pymoo: Multi-Objective Optimization in Python.
\newblock \emph{IEEE Access}, 8: 89497--89509.

\bibitem[{Burkart and Huber(2021)}]{burkart2021survey}
Burkart, N.; and Huber, M.~F. 2021.
\newblock A survey on the explainability of supervised machine learning.
\newblock \emph{Journal of Artificial Intelligence Research}, 70: 245--317.

\bibitem[{Chang and Wild(2023{\natexlab{a}})}]{chang2023designing}
Chang, T.~H.; and Wild, S.~M. 2023{\natexlab{a}}.
\newblock Designing a Framework for Solving Multiobjective Simulation Optimization Problems.
\newblock \emph{arXiv preprint arXiv:2304.06881}.

\bibitem[{Chang and Wild(2023{\natexlab{b}})}]{chang2023parmoo}
Chang, T.~H.; and Wild, S.~M. 2023{\natexlab{b}}.
\newblock {ParMOO}: A Python library for parallel multiobjective simulation optimization.
\newblock \emph{Journal of Open Source Software}, 8(82): 4468.

\bibitem[{Chugh(2020)}]{chugh2020scalarizing}
Chugh, T. 2020.
\newblock Scalarizing functions in Bayesian multiobjective optimization.
\newblock In \emph{2020 IEEE Congress on Evolutionary Computation (CEC)}, 1--8. IEEE.

\bibitem[{Cocchi and Lapucci(2020)}]{cocchiandlapucci2020}
Cocchi, G.; and Lapucci, M. 2020.
\newblock An augmented Lagrangian algorithm for multi-objective optimization.
\newblock \emph{Computational Optimization and Applications}, 77(1): 29--56.

\bibitem[{Das and Dennis(1998)}]{das1998}
Das, I.; and Dennis, J.~E. 1998.
\newblock Normal-boundary intersection: A new method for generating the {P}areto surface in nonlinear multicriteria optimization problems.
\newblock \emph{SIAM Journal on Optimization}, 8(3): 631--657.

\bibitem[{Daulton, Balandat, and Bakshy(2020)}]{daulton2020differentiable}
Daulton, S.; Balandat, M.; and Bakshy, E. 2020.
\newblock Differentiable expected hypervolume improvement for parallel multi-objective Bayesian optimization.
\newblock \emph{Advances in Neural Information Processing Systems}, 33: 9851--9864.

\bibitem[{Deb and Jain(2013)}]{deb2013evolutionary}
Deb, K.; and Jain, H. 2013.
\newblock An evolutionary many-objective optimization algorithm using reference-point-based nondominated sorting approach, part {I}: solving problems with box constraints.
\newblock \emph{IEEE transactions on evolutionary computation}, 18(4): 577--601.

\bibitem[{Deb et~al.(2002)Deb, Pratap, Agarwal, and Meyarivan}]{deb2002fast}
Deb, K.; Pratap, A.; Agarwal, S.; and Meyarivan, T. 2002.
\newblock A fast and elitist multiobjective genetic algorithm: NSGA-II.
\newblock \emph{IEEE transactions on evolutionary computation}, 6(2): 182--197.

\bibitem[{Deb et~al.(2005)Deb, Thiele, Laumanns, and Zitzler}]{deb2005scalable}
Deb, K.; Thiele, L.; Laumanns, M.; and Zitzler, E. 2005.
\newblock Scalable test problems for evolutionary multiobjective optimization.
\newblock In \emph{Evolutionary Multiobjective Optimization, Theoretical Advances and Applications}, chapter~6. London, UK: Springer.

\bibitem[{{DeepHyper Development Team}(2018)}]{deephyper_software}
{DeepHyper Development Team}. 2018.
\newblock "DeepHyper: A Python Package for Scalable Neural Architecture and Hyperparameter Search".

\bibitem[{Deshpande, Watson, and Canfield(2016)}]{deshpande2016multiobjective}
Deshpande, S.; Watson, L.~T.; and Canfield, R.~A. 2016.
\newblock Multiobjective optimization using an adaptive weighting scheme.
\newblock \emph{Optimization Methods and Software}, 31(1): 110--133.

\bibitem[{Dwork(2008)}]{dwork2008differential}
Dwork, C. 2008.
\newblock Differential privacy: A survey of results.
\newblock In \emph{International conference on theory and applications of models of computation}, 1--19. Springer.

\bibitem[{Egele et~al.(2020)Egele, Balaprakash, Vishwanath, Guyon, and Liu}]{egele_agebo-tabular_2020}
Egele, R.; Balaprakash, P.; Vishwanath, V.; Guyon, I.; and Liu, Z. 2020.
\newblock {AgEBO}-{Tabular}: {Joint} {Neural} {Architecture} and {Hyperparameter} {Search} with {Autotuned} {Data}-{Parallel} {Training} for {Tabular} {Data}.
\newblock \emph{arXiv:2010.16358 [cs, stat]}.
\newblock ArXiv: 2010.16358.

\bibitem[{Egele et~al.(2022{\natexlab{a}})Egele, Gouneau, Vishwanath, Guyon, and Balaprakash}]{egele2022asynchronous}
Egele, R.; Gouneau, J.; Vishwanath, V.; Guyon, I.; and Balaprakash, P. 2022{\natexlab{a}}.
\newblock Asynchronous Distributed Bayesian Optimization at HPC Scale.
\newblock \emph{arXiv preprint arXiv:2207.00479}.

\bibitem[{Egele et~al.(2022{\natexlab{b}})Egele, Gouneau, Vishwanath, Guyon, and Balaprakash}]{egele_asynchronous_2022}
Egele, R.; Gouneau, J.; Vishwanath, V.; Guyon, I.; and Balaprakash, P. 2022{\natexlab{b}}.
\newblock Asynchronous {Distributed} {Bayesian} {Optimization} at {HPC} {Scale}.
\newblock ArXiv:2207.00479 [cs].

\bibitem[{Eggensperger et~al.(2021)Eggensperger, M{\"{u}}ller, Mallik, Feurer, Sass, Klein, Awad, Lindauer, and Hutter}]{eggensperger_hpobench_2021}
Eggensperger, K.; M{\"{u}}ller, P.; Mallik, N.; Feurer, M.; Sass, R.; Klein, A.; Awad, N.~H.; Lindauer, M.; and Hutter, F. 2021.
\newblock HPOBench: {A} Collection of Reproducible Multi-Fidelity Benchmark Problems for {HPO}.
\newblock \emph{CoRR}, abs/2109.06716.

\bibitem[{Ehrgott(2005)}]{ehrgott2005multicriteria}
Ehrgott, M. 2005.
\newblock \emph{Multicriteria optimization}, volume 491.
\newblock Springer Science \& Business Media.

\bibitem[{Gawlikowski et~al.(2023)Gawlikowski, Tassi, Ali, Lee, Humt, Feng, Kruspe, Triebel, Jung, Roscher et~al.}]{gawlikowski2023survey}
Gawlikowski, J.; Tassi, C. R.~N.; Ali, M.; Lee, J.; Humt, M.; Feng, J.; Kruspe, A.; Triebel, R.; Jung, P.; Roscher, R.; et~al. 2023.
\newblock A survey of uncertainty in deep neural networks.
\newblock \emph{Artificial Intelligence Review}, 1--77.

\bibitem[{Hutter(2012)}]{smac2012}
Hutter, F. e.~a. 2012.
\newblock Parallel algorithm configuration.
\newblock In \emph{International Conference on Learning and Intelligent Optimization}, 55--70. Springer.

\bibitem[{Ishibuchi et~al.(2015)Ishibuchi, Masuda, Tanigaki, and Nojima}]{ishibuchi2015modified}
Ishibuchi, H.; Masuda, H.; Tanigaki, Y.; and Nojima, Y. 2015.
\newblock Modified distance calculation in generational distance and inverted generational distance.
\newblock In \emph{Evolutionary Multi-Criterion Optimization: 8th International Conference, EMO 2015, Guimar{\~a}es, Portugal, March 29--April 1, 2015. Proceedings, Part II 8}, 110--125. Springer.

\bibitem[{Jordan and Mitchell(2015)}]{jordan2015machine}
Jordan, M.~I.; and Mitchell, T.~M. 2015.
\newblock Machine learning: Trends, perspectives, and prospects.
\newblock \emph{Science}, 349(6245): 255--260.

\bibitem[{Karl et~al.(2022)Karl, Pielok, Moosbauer, Pfisterer, Coors, Binder, Schneider, Thomas, Richter, Lang, Garrido-Merchán, Branke, and Bischl}]{karl2022multiobjective}
Karl, F.; Pielok, T.; Moosbauer, J.; Pfisterer, F.; Coors, S.; Binder, M.; Schneider, L.; Thomas, J.; Richter, J.; Lang, M.; Garrido-Merchán, E.~C.; Branke, J.; and Bischl, B. 2022.
\newblock Multi-Objective Hyperparameter Optimization -- An Overview.
\newblock arXiv:2206.07438.

\bibitem[{Klein and Hutter(2019)}]{klein_tabular_2019}
Klein, A.; and Hutter, F. 2019.
\newblock Tabular Benchmarks for Joint Architecture and Hyperparameter Optimization.
\newblock \emph{CoRR}, abs/1905.04970.

\bibitem[{Knowles(2006)}]{parego2006}
Knowles, J. 2006.
\newblock ParEGO: a hybrid algorithm with on-line landscape approximation for expensive multiobjective optimization problems.
\newblock \emph{IEEE Transactions on Evolutionary Computation}, 10(1): 50--66.

\bibitem[{{Le Digabel} and Wild(2015)}]{ledigabel2015}
{Le Digabel}, S.; and Wild, S.~M. 2015.
\newblock A Taxonomy of Constraints in Simulation-Based Optimization.
\newblock Preprint ANL/MCS-P5350-0515, Argonne National Laboratory, Mathematics and Computer Science Division.

\bibitem[{Lindauer et~al.(2022)Lindauer, Eggensperger, Feurer, Biedenkapp, Deng, Benjamins, Ruhkopf, Sass, and Hutter}]{smac3_2022}
Lindauer, M.; Eggensperger, K.; Feurer, M.; Biedenkapp, A.; Deng, D.; Benjamins, C.; Ruhkopf, T.; Sass, R.; and Hutter, F. 2022.
\newblock SMAC3: A Versatile Bayesian Optimization Package for Hyperparameter Optimization.
\newblock \emph{Journal of Machine Learning Research}, 23(54): 1--9.

\bibitem[{Liu et~al.(2022)Liu, Pavao, Xu, Escalera, Ferreira, Guyon, Hong, Hutter, Ji, J{\'{u}}nior, Li, Lindauer, Luo, Madadi, Nierhoff, Niu, Pan, Stoll, Treguer, Wang, Wang, Wu, Xiong, Zela, and Zhang}]{autdl_challenge_2019}
Liu, Z.; Pavao, A.; Xu, Z.; Escalera, S.; Ferreira, F.; Guyon, I.; Hong, S.; Hutter, F.; Ji, R.; J{\'{u}}nior, J. C. S.~J.; Li, G.; Lindauer, M.; Luo, Z.; Madadi, M.; Nierhoff, T.; Niu, K.; Pan, C.; Stoll, D.; Treguer, S.; Wang, J.; Wang, P.; Wu, C.; Xiong, Y.; Zela, A.; and Zhang, Y. 2022.
\newblock Winning solutions and post-challenge analyses of the ChaLearn AutoDL challenge 2019.
\newblock \emph{CoRR}, abs/2201.03801.

\bibitem[{M{\"a}rtens and Izzo(2013)}]{martens2013asynchronous}
M{\"a}rtens, M.; and Izzo, D. 2013.
\newblock The asynchronous island model and NSGA-II: study of a new migration operator and its performance.
\newblock In \emph{Proceedings of the 15th annual conference on Genetic and evolutionary computation}, 1173--1180.

\bibitem[{Mehrabi et~al.(2021)Mehrabi, Morstatter, Saxena, Lerman, and Galstyan}]{mehrabi2021survey}
Mehrabi, N.; Morstatter, F.; Saxena, N.; Lerman, K.; and Galstyan, A. 2021.
\newblock A survey on bias and fairness in machine learning.
\newblock \emph{ACM computing surveys (CSUR)}, 54(6): 1--35.

\bibitem[{Muhammad and Bae(2022)}]{muhammad2022survey}
Muhammad, A.; and Bae, S.-H. 2022.
\newblock A survey on efficient methods for adversarial robustness.
\newblock \emph{IEEE Access}, 10: 118815--118830.

\bibitem[{{Optuna developers}(2018)}]{optunadocs}
{Optuna developers}. 2018.
\newblock Optuna API: optuna.samplers.
\newblock \url{https://optuna.readthedocs.io/en/stable/reference/samplers/index.html}.
\newblock Accessed: Jun 27, 2023.

\bibitem[{Ozaki et~al.(2022)Ozaki, Tanigaki, Watanabe, Nomura, and Onishi}]{ozaki2022multiobjective}
Ozaki, Y.; Tanigaki, Y.; Watanabe, S.; Nomura, M.; and Onishi, M. 2022.
\newblock Multiobjective tree-structured Parzen estimator.
\newblock \emph{Journal of Artificial Intelligence Research}, 73: 1209--1250.

\bibitem[{Ozaki et~al.(2020)Ozaki, Tanigaki, Watanabe, and Onishi}]{ozaki2020multiobjective}
Ozaki, Y.; Tanigaki, Y.; Watanabe, S.; and Onishi, M. 2020.
\newblock Multiobjective Tree-Structured Parzen Estimator for Computationally Expensive Optimization Problems.
\newblock In \emph{Proc. the 2020 Genetic and Evolutionary Computation Conference (GECCO '20)}, 533–541.

\bibitem[{Pfisterer et~al.(2022)Pfisterer, Schneider, Moosbauer, Binder, and Bischl}]{pfisterer2022yahpo}
Pfisterer, F.; Schneider, L.; Moosbauer, J.; Binder, M.; and Bischl, B. 2022.
\newblock Yahpo gym-an efficient multi-objective multi-fidelity benchmark for hyperparameter optimization.
\newblock In \emph{International Conference on Automated Machine Learning}, 3--1. PMLR.

\bibitem[{Pinelis(2019)}]{pinelis2019order}
Pinelis, I. 2019.
\newblock Order statistics on the spacings between order statistics for the uniform distribution.
\newblock \emph{arXiv preprint arXiv:1909.06406}.

\bibitem[{Song et~al.(2021)Song, Perello-Nieto, Santos-Rodriguez, Kull, Flach et~al.}]{song2021classifier}
Song, H.; Perello-Nieto, M.; Santos-Rodriguez, R.; Kull, M.; Flach, P.; et~al. 2021.
\newblock Classifier Calibration: A survey on how to assess and improve predicted class probabilities.
\newblock \emph{arXiv preprint arXiv:2112.10327}.

\bibitem[{Tan et~al.(2019)Tan, Chen, Pang, Vasudevan, Sandler, Howard, and Le}]{tan2019mnasnet}
Tan, M.; Chen, B.; Pang, R.; Vasudevan, V.; Sandler, M.; Howard, A.; and Le, Q.~V. 2019.
\newblock Mnasnet: Platform-aware neural architecture search for mobile.
\newblock In \emph{Proceedings of the IEEE/CVF conference on computer vision and pattern recognition}, 2820--2828.

\bibitem[{Tjoa and Guan(2020)}]{tjoa2020survey}
Tjoa, E.; and Guan, C. 2020.
\newblock A survey on explainable artificial intelligence (xai): Toward medical xai.
\newblock \emph{IEEE transactions on neural networks and learning systems}, 32(11): 4793--4813.

\bibitem[{Xia et~al.(2018)Xia, Shukla, Brettin, Garcia-Cardona, Cohn, Allen, Maslov, Holbeck, Doroshow, Evrard, Stahlberg, and Stevens}]{xia_predicting_2018}
Xia, F.; Shukla, M.; Brettin, T.; Garcia-Cardona, C.; Cohn, J.; Allen, J.~E.; Maslov, S.; Holbeck, S.~L.; Doroshow, J.~H.; Evrard, Y.~A.; Stahlberg, E.~A.; and Stevens, R.~L. 2018.
\newblock Predicting tumor cell line response to drug pairs with deep learning.
\newblock \emph{BMC Bioinformatics}, 19(S18): 486.

\end{thebibliography}

\newpage

\appendix

\section{Decentralized Bayesian Optimization Algorithm}

This section provides a detailed description of the algorithm used for parallel Bayesian optimization. This asynchronous decentralized scheme we use was proposed by~\cite{egele2022asynchronous} and is detailed in Algorithm~\ref{alg:dbo-algorithm}. This scheme allows for efficient resource utilization with an \emph{overhead for new  ``suggest'' that is not impacted by the number of parallel workers} and exploration-exploitation trade-off through a periodic exponential decay scheduler. Indeed, exploration can become excessive when increasing the number of parallel resources. In-depth, each process starts by initializing the local random states with $\texttt{seed\_global}$ (line 1). From this state, the initial local exploration-exploitation trade-off is sampled from an exponential distribution of mean parameter $\frac{1}{\kappa}$ (line 2). Then, the random state is updated with a new $\texttt{seed\_local}$ (line 3-5). The optimization agent and share memory are initialized (lines 5-6).
After initialization, the optimization loop is entered. Each loop iteration starts by updating $\kappa_t$ following the periodic decay scheduler (line 9). Next, the optimizer is queried for a new configuration $\texttt{x}$ (line 10) which is used to evaluate the black-box function $\texttt{func}$ (line 11). Once, the evaluation is completed, new observations from other processes working with the same shared memory $\texttt{storage}$ are retrieved (line 12), the latest observation is added to the shared memory (line 13), and all new observations are concatenated together $\texttt{X},\texttt{Y}$ (lines 14-15). The optimizer is finally updated with this new batch of observations (line 17).

\begin{algorithm2e}[!t]
\small
\SetInd{0.25em}{0.5em}
\SetAlgoLined
\SetKwInOut{Input}{Inputs}\SetKwInOut{Output}{Output}
\SetKwFunction{Tell}{observe}
\SetKwFunction{Ask}{suggest}
\SetKwFunction{Exp}{Exp}
\SetKwFunction{writef}{write}
\SetKwFunction{readf}{read}

\SetKwFor{For}{for}{do}{end}

\Input{$\texttt{func}$: black-box function, $\kappa$: $\lcb$ hyperparameter, $\texttt{R}$: rank of the process, $\texttt{Np}$: number of processes, $\texttt{seed\_global}$: global random seed, $\texttt{T}$: period of the exponential decay, $\lambda$: decay rate of the exponential decay}
\Output{$\texttt{storage}$ all evaluated configurations}
{\color{orange}\tcc{Initialization}}
Initialize the random generator with $\texttt{seed\_global}$ \;
$\kappa_0 \gets $ Sample new local trade-off from $\texttt{Exp}(\frac{1}{\kappa})$ \; 
$\texttt{seed\_local} \gets $ Get index $\texttt{R}$ of $\texttt{Np}$ sampled integers \;
Initialize the random generator with $\texttt{seed\_local}$ \;
$\texttt{optimizer} \gets $ New Bayesian Optimizer \;
$\texttt{storage} \gets $ New shared memory \;
$\texttt{t} \gets 0$\;
{\color{orange}\tcc{Optimization Loop}}
\While{not done}{
    {\color{blue}\tcc{Apply exponential decay}}
    $\kappa_t \leftarrow \kappa_0 \cdot \exp(-\lambda \cdot (\texttt{t} \bmod \texttt{T}))$ \;
    {\color{blue}\tcc{Query and Evaluate Configuration}}
    $\texttt{x} \gets \texttt{optimizer}$.\Ask{$\kappa_t$} \;
    $\texttt{y} \gets$ Evaluate $\texttt{func}(\texttt{x})$ \;
    {\color{blue}\tcc{Write/Read From Shared-Memory}}
    $\texttt{X}, \texttt{Y} \gets$ Read new observations from $\texttt{storage}$ \;
    Write $(\texttt{x}, \texttt{y})$ to $\texttt{storage}$ \;
    $\texttt{X} \gets $ Concatenate $\texttt{X}$ with $[\texttt{x}]$ \;
    $\texttt{Y} \gets $ Concatenate $\texttt{Y}$ with $[\texttt{y}]$ \;
    $\texttt{t} \leftarrow \texttt{t} + 1$ \;
    {\color{blue}\tcc{Update the Model}}
    $\texttt{optimizer}.$\Tell{$\texttt{X},\texttt{Y}$} \;
}
\Return \texttt{storage} \;

\caption{Decentralized Bayesian Optimization}
\label{alg:dbo-algorithm}
\end{algorithm2e}

\section{Details on Scalarization Functions}
\label{appendix:scalarization-function}

In this section, we provide more details about the set of scalarization functions we considered. Let $w$ be a weight vector of dimension $N_o$ such that $|w| = 1$ and $w_i \geq 0$, let $y \in \cY$ be an objective vector also of dimension $N_o$, and let $z$ be the \emph{utopia point} given by the objectives minimized independently.

The weighted-sum or \emph{linear} (L) scalarization  is:
\begin{equation}
    \label{eq:linear-scalarization}
    s^{L}_w(y) = \sum_{i=1}^{N_o} w_i y_i
\end{equation}
The \emph{Chebyshev} (CH) scalarization is:
\begin{equation}
    \label{eq:chebyshev-scalarization}
    s^{CH}_{w,z}(y) = \max_{i\in[1;N_o]} w_i|y_i - z_i|
\end{equation}
% The \emph{augmented Chebyshev} (AC) scalarization is:
% \begin{equation}
%     \label{eq:augchebyshev-scalarization}
%     s^{AC}_{w,z}(y) = \alpha\sum_{i=1}^{N_o}|y_i - z_i| + \max_{i\in[1;N_o]} w_i|y_i - z_i|
% \end{equation}
% with $\alpha \in \mathbb{R}$ a parameter (default to 0.001).
The \emph{penalty-boundary intersection} (PBI) scalarization is:
\begin{equation}
    \label{eq:pbi-scalarization}
    s^{PBI}_{w,z}(y) = d_1(z - y; w) + \theta \cdot d_2(z - y; w)
\end{equation}
where $d_1(y; w) = |y^\top w / \|w\||, d_2(y; w) = \|y - d_1 w / \|w\|\|$ and $\theta \in \mathbb{R}$ a parameter (default to 5).
% The \emph{quadratic weighted-sum} (Q) is:
% \begin{equation}
%     \label{eq:quadratic}
%     s^{Q}_{w,z}(y) = (y-z)^\top \cdot Q(w) \cdot (y-z)
% \end{equation}
% such that $Q(w) = UAU^\top$ where $U$ is an orthonormal matrix given by the singular value decomposition of $w = U S V$, and $A = \text{diag}(1, \alpha, \ldots, \alpha) \in \mathbb{R}^{N_o}$ with $\alpha \in \mathbb{R}$ (default to 10).

\section{Additional Results on HPOBench}

In this section, we provide the HVI curves of the experiments from which we derived the ranking. Figure~\ref{fig:hpobench_average_hypervolume} presents these results. The same conclusion as Figure~\ref{fig:hpobench_average_ranking} can be derived. For example, we clearly see that NSGAII is progressively closing the gap with other competitors. We can also see that the (PBI) function is more robust than (CH) for MML normalization.

\begin{figure}[!h]
    \centering
    \includegraphics[width=\columnwidth]{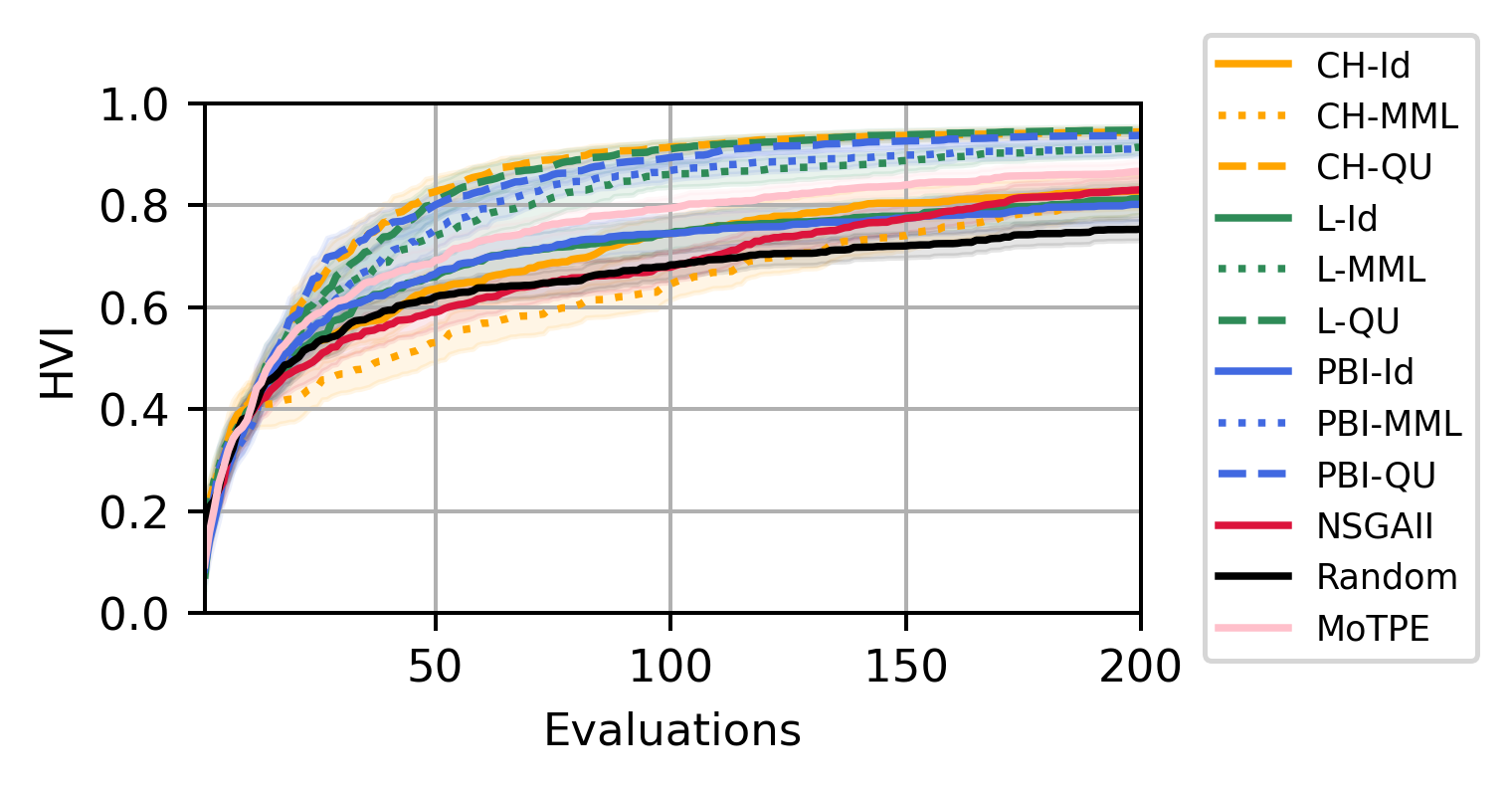}
    \caption{HVI curves of scalarization functions and objective normalization combinations on the HPOBench tabular tasks (i.e., 40 experiments per curve).}
    \label{fig:hpobench_average_hypervolume}
\end{figure}

\section{Additional Results on DTLZ Benchmark}
\label{appendix:dtlz}

To complement the HPO benchmark, we also evaluated our method on the well-known DTLZ benchmark suite, which consists of 7 continuous multi-objective optimization problems \cite{deb2005scalable}.
The problems can be configured to have any number of design variables ($N_i$) and objectives ($N_o$).
Note that each of the DTLZ problems is carefully constructed to be pathologically difficult in some way.
In particular, each of the problems exhibits one or more of the following properties
\begin{itemize}
    \item \textbf{P1}: local Pareto optimal solutions.
    \item \textbf{P2}: non-convexity of the PF, which causes many common scalarization functions such as linear (L)to converge to a small number of solutions (``corners'') leaving most of the PF unexplored.
    \item \textbf{P3}: degeneracy of the PF, making it less than $(N_o-1)$-dimensional.
    \item \textbf{P4}: non-uniform density of PF such as a larger concentration on the ``edges''.
    \item \textbf{P5}: a disjoint PF (i.e., formed by a union of non-intersecting PFs).
\end{itemize}
Table~\ref{tab:properties-dtlz} summarizes the properties of the DTLZ problems.

\begin{table}[]
    \centering
    \resizebox{\columnwidth}{!}{%
    \begin{tabular}{c|c|c|c|c|c|}
    \cline{2-6}
     &
      \begin{tabular}[c]{@{}c@{}}\textbf{P1}\\ Local\\ PFs\end{tabular} &
      \begin{tabular}[c]{@{}c@{}}\textbf{P2}\\ Non-Convex\\ PF\end{tabular} &
      \begin{tabular}[c]{@{}c@{}}\textbf{P3}\\ Degenerated\\ PF\end{tabular} &
      \begin{tabular}[c]{@{}c@{}}\textbf{P4}\\ Non-Uniform\\ Density of PF\end{tabular} &
      \begin{tabular}[c]{@{}c@{}}\textbf{P5}\\ Disjoint\\ PF\end{tabular} \\ \hline
    \multicolumn{1}{|c|}{\textbf{DTLZ 1}}                                                           & \checkmark &   &   &   &   \\ \hline
    \multicolumn{1}{|c|}{\textbf{DTLZ 2}}                                                           &   & \checkmark &   &   &   \\ \hline
    \multicolumn{1}{|c|}{\begin{tabular}[c]{@{}c@{}}\textbf{DTLZ 3}\\ (harder 2)\end{tabular}} & \checkmark & \checkmark &   &   &   \\ \hline
    \multicolumn{1}{|c|}{\textbf{DTLZ 4}}                                                           &   & \checkmark &   & \checkmark &   \\ \hline
    \multicolumn{1}{|c|}{\textbf{DTLZ 5}}                                                           &   &  & \checkmark &   &   \\ \hline
    \multicolumn{1}{|c|}{\begin{tabular}[c]{@{}c@{}}\textbf{DTLZ 6}\\ (harder 5)\end{tabular}} &  &  & \checkmark &   &   \\ \hline
    \multicolumn{1}{|c|}{\textbf{DTLZ 7}}                                                           &  & \checkmark  &   &   & \checkmark \\ \hline
    \end{tabular}%
    }
    \caption{Properties of DTLZ problems.}
    \label{tab:properties-dtlz}
\end{table}

In all cases, we define instances of the problems with $N_i=8$ variables and $N_o=3$ objectives. Then, we perform 10 repetitions for all combinations of the scalarizations from above (i.e., Linear (L), Chebyshev (CH), and PBI) with objective normalizations (i.e., Identity (Id), MinMax-Log (MML), Quantile-Uniform (QU)). We also compare against the NSGAII implementation from Optuna (and also pymoo~\cite{blank2020pymoo} but as the results are similar we only present Optuna which is asynchronous parallel) and a two-phase solver that uses 2000 initial evaluations of the global design of experiments followed by a local trust-region solve off the best points, built with ParMOO by following the high-dimensional MOO tutorial in the docs \cite{chang2023parmoo}. MoTPE was not used because it is excessively slow in this setup. We evaluate our results by measuring the HVI (as a diversity metric) and $\GD$ (as a convergence metric). 
To summarize the results across all experiments we provide the average ranking of evaluated optimizers in Figure~\ref{fig:dtlz}. From this average ranking we removed DTLZ 1 and 3 which appear to not be resolved by any of the considered optimizers as shown in Figures~\ref{fig:dtlz1-hv} and \ref{fig:dtlz3-hv} (both with local Pareto-optimal solutions). With 10 repetitions and 5 problems, the average ranking is based on 50 experiments. Individual HVI and $\GD$ metrics with respect to the number of evaluated black-box functions are provided in Figures~\ref{fig:dtlz-plots-part-1} and \ref{fig:dtlz-plots-part-2}.

Although the problems exhibit certain pathological properties which can happen in MOHPO, none of them exhibit differences in objective scales or large objective ranges therefore limiting the need for objective normalization. In fact, it can be observed in the average ranking (Figure~\ref{fig:dtlz}) that for D-MoBO the best performing combination is based on identity (Id) normalization (i.e., no normalization is performed). Also, looking at the Figures~|\ref{fig:dtlz-plots-part-1} and \ref{fig:dtlz-plots-part-2} of the independent problems it is clear that there is a joint effect between scalarization functions and objectives normalization which makes it impossible to fix one objectives normalization for all scalarization functions.

\begin{figure*} %[!t]
    \centering
    \begin{subfigure}{\columnwidth}
        \centering
        \includegraphics[width=\textwidth]{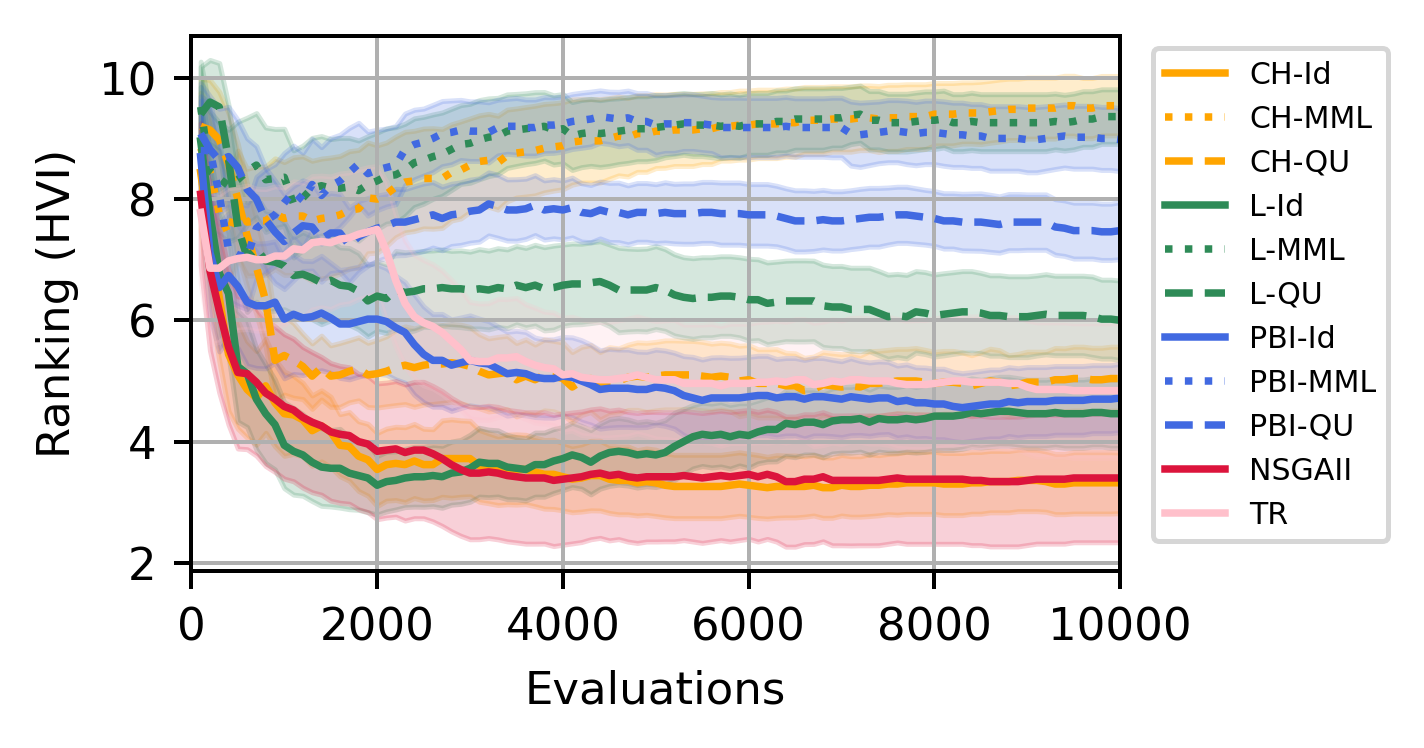}
        \caption{Ranking from HVI}
        \label{fig:dtlz_rank_from_hv}
    \end{subfigure}
    \begin{subfigure}{\columnwidth}
        \centering
        \includegraphics[width=\textwidth]{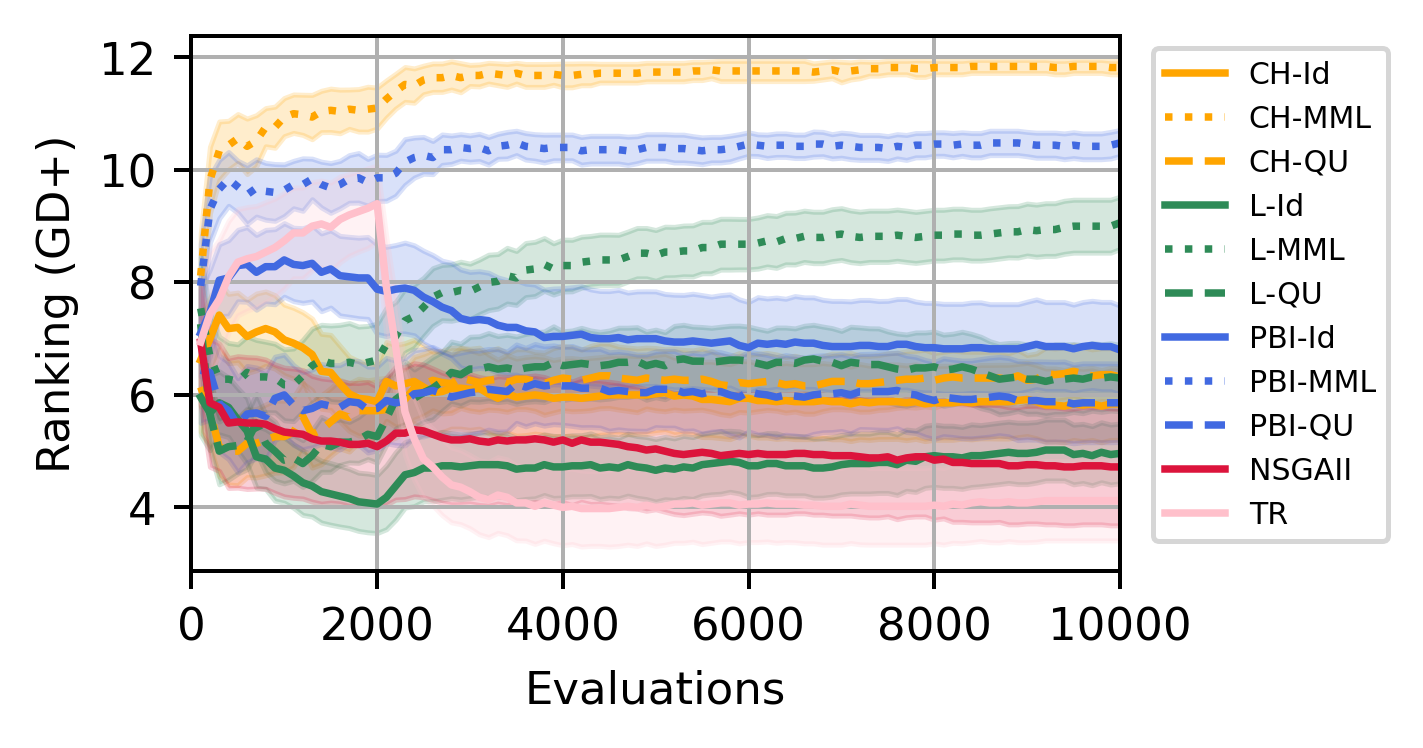}
        \caption{Ranking from $\GD$}
        \label{fig:dtlz_rank_from_gd}
    \end{subfigure}
    \caption{Average ranking (± 1.96 standard error) with respect to the number of black-box function evaluations of D-MoBO (with different scalarization functions and objectives normalization) and other MOO optimizers on the continuous DTLZ benchmark (2,4,5,6 and 7).}
    \label{fig:dtlz}
\end{figure*}

\begin{figure*} %[!t]
    \centering

    \begin{subfigure}{\columnwidth}
        \centering
        \includegraphics[width=\textwidth]{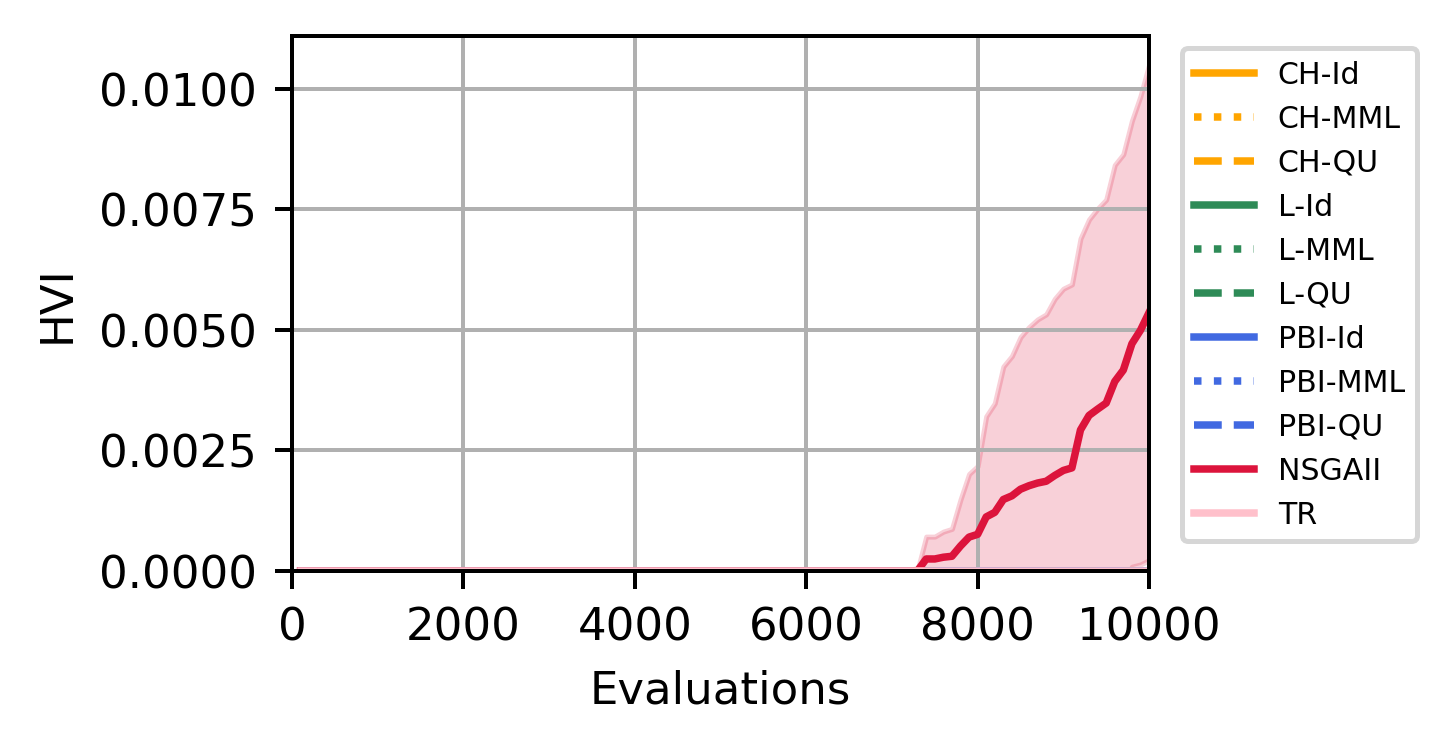}
        \caption{DTLZ 1 (HVI) --- \emph{unresolved}}
        \label{fig:dtlz1-hv}
    \end{subfigure}
    \begin{subfigure}{\columnwidth}
        \centering
        \includegraphics[width=\textwidth]{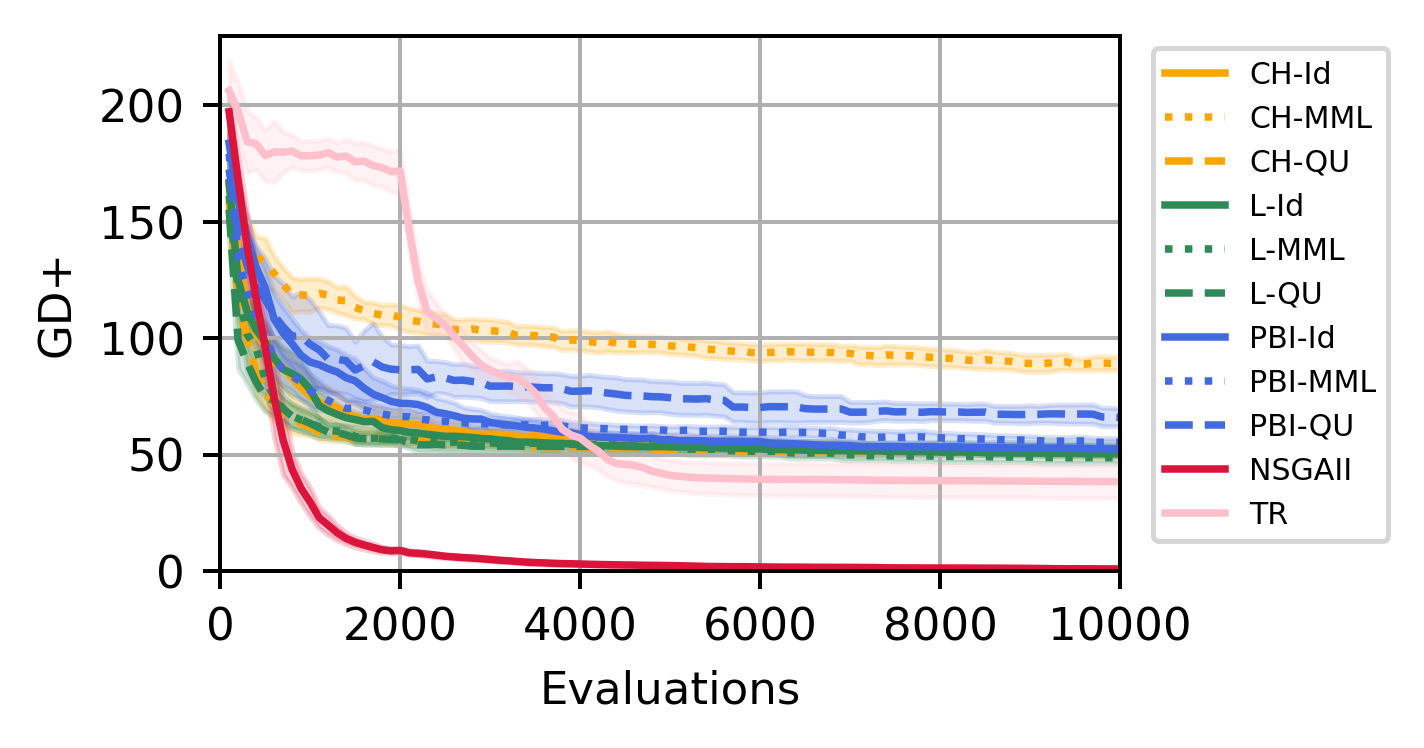}
        \caption{DTLZ 1 ($\GD$) --- \emph{unresolved}}
        \label{fig:dtlz1-gd}
    \end{subfigure}

    \begin{subfigure}{\columnwidth}
        \centering
        \includegraphics[width=\textwidth]{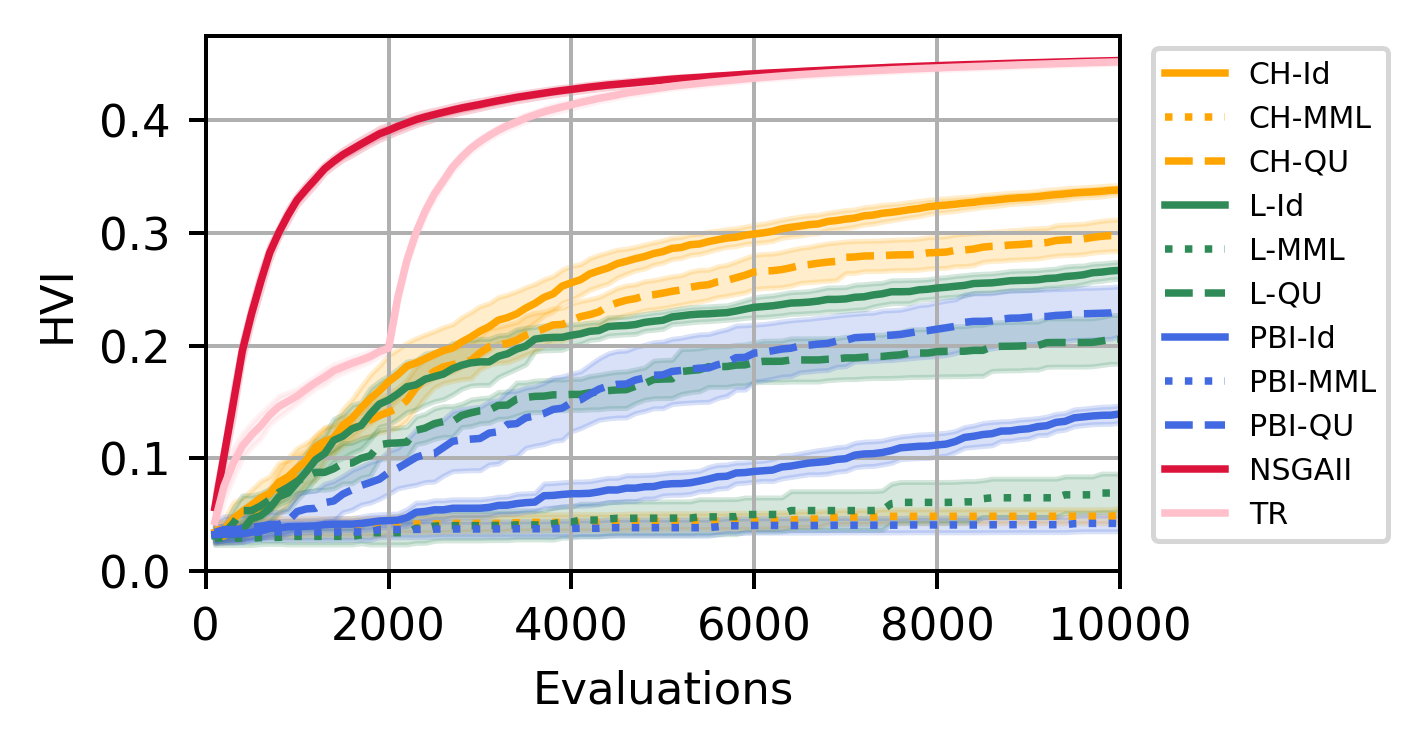}
        \caption{DTLZ 2 (HVI)}
        \label{fig:dtlz2-hv}
    \end{subfigure}
    \begin{subfigure}{\columnwidth}
        \centering
        \includegraphics[width=\textwidth]{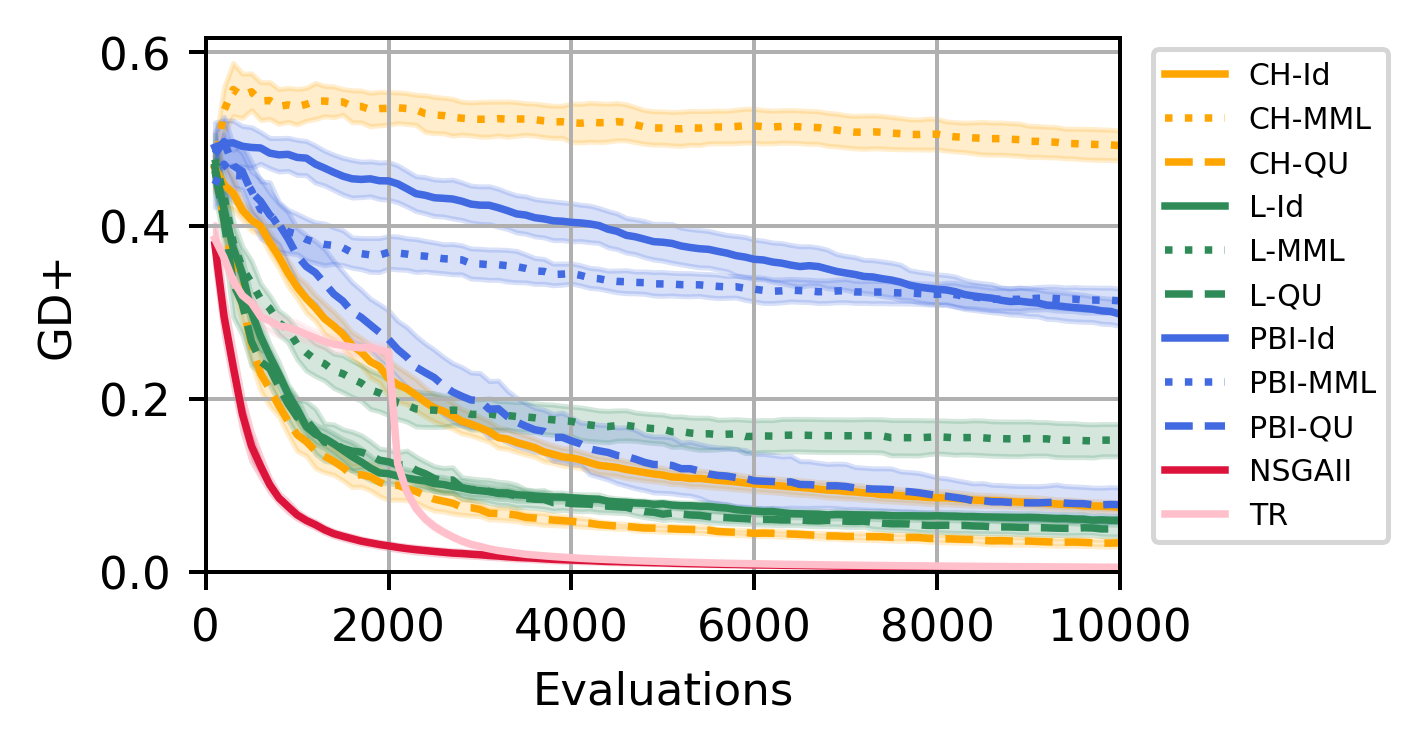}
        \caption{DTLZ 2 ($\GD$)}
        \label{fig:dtlz2-gd}
    \end{subfigure}

    \begin{subfigure}{\columnwidth}
        \centering
        \includegraphics[width=\textwidth]{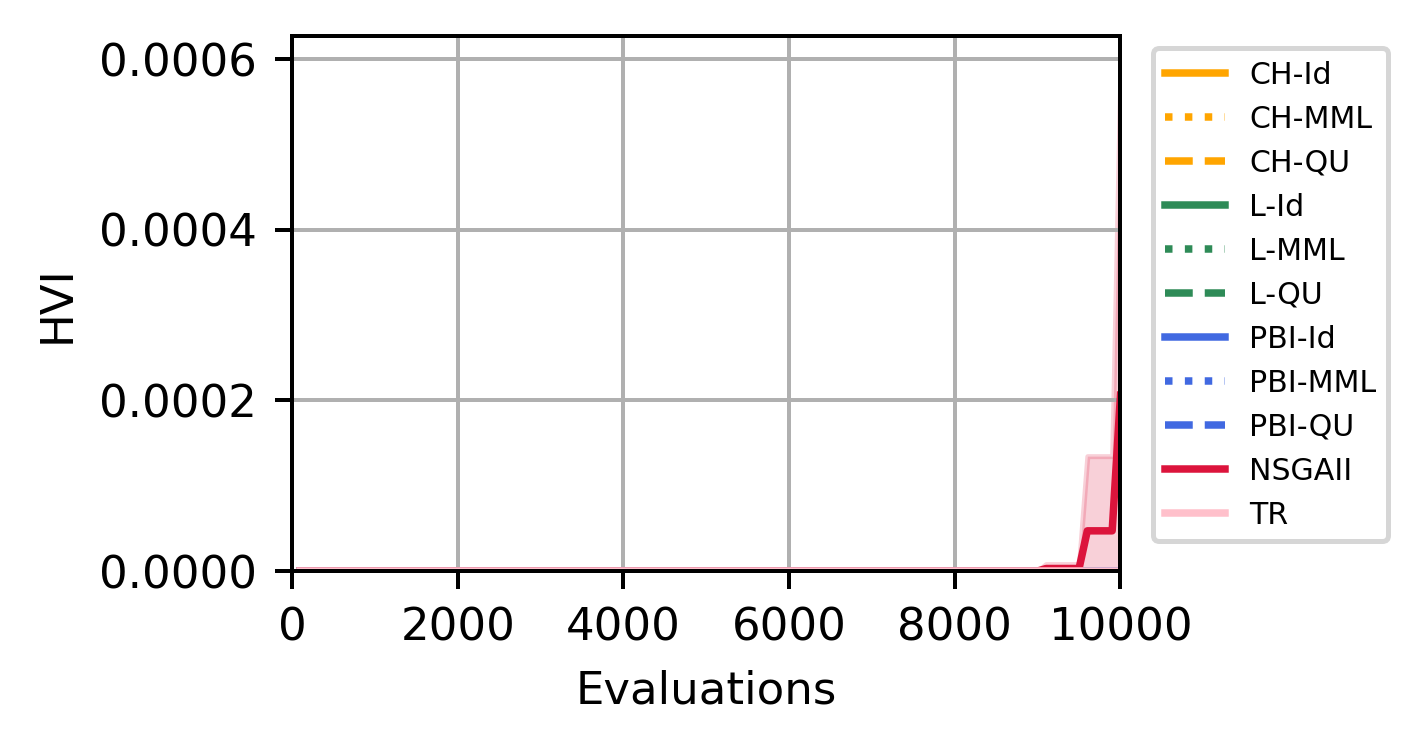}
        \caption{DTLZ 3 (HVI) --- \emph{unresolved}}
        \label{fig:dtlz3-hv}
    \end{subfigure}
    \begin{subfigure}{\columnwidth}
        \centering
        \includegraphics[width=\textwidth]{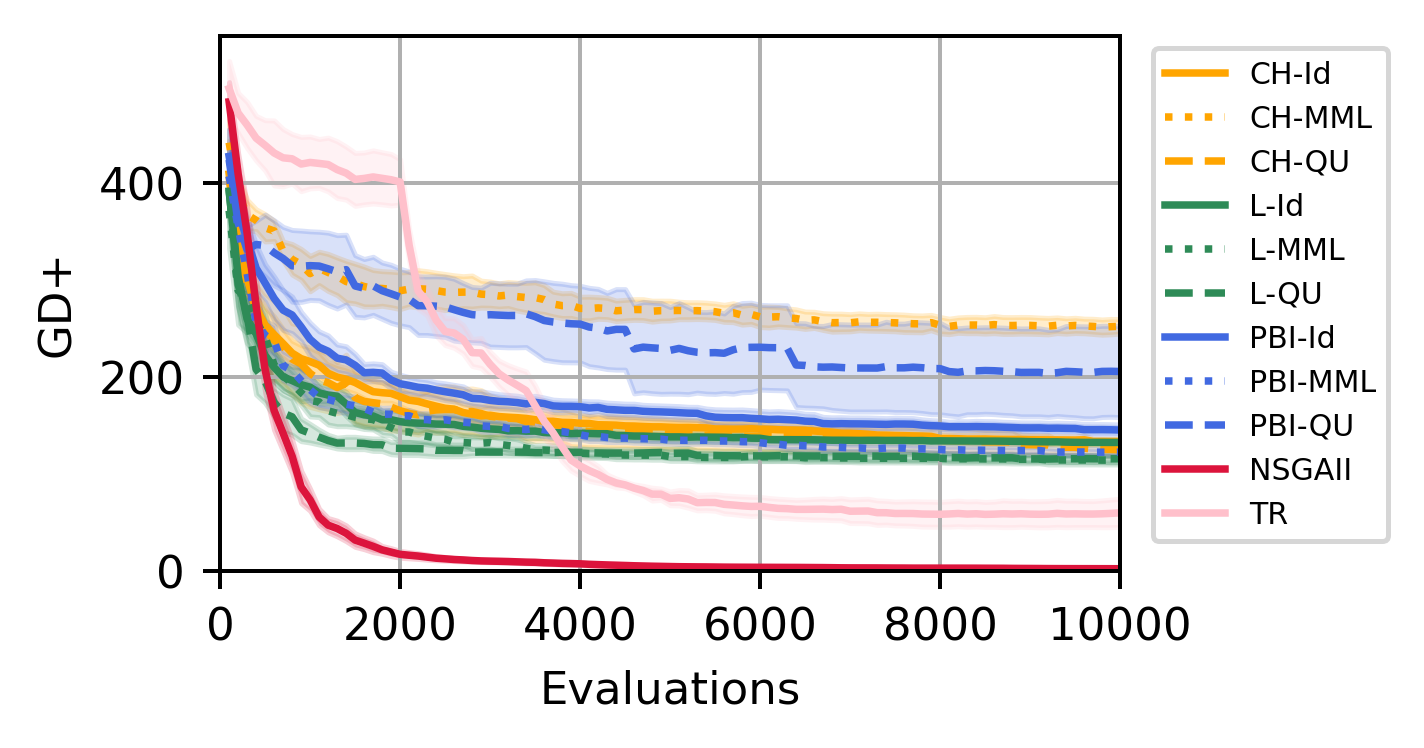}
        \caption{DTLZ 3 ($\GD$) --- \emph{unresolved}}
        \label{fig:dtlz3-gd}
    \end{subfigure}

    \begin{subfigure}{\columnwidth}
        \centering
        \includegraphics[width=\textwidth]{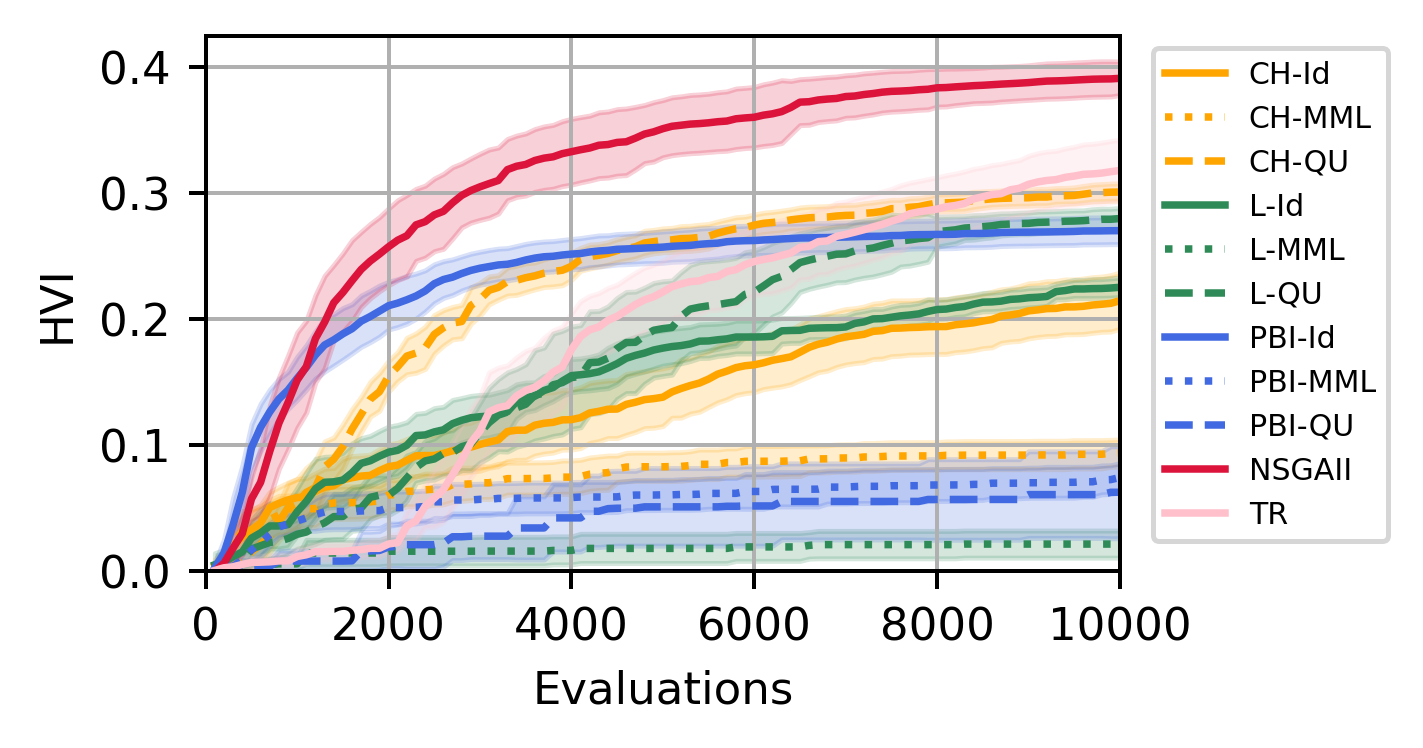}
        \caption{DTLZ 4 (HVI)}
        \label{fig:dtlz4-hv}
    \end{subfigure}
    \begin{subfigure}{\columnwidth}
        \centering
        \includegraphics[width=\textwidth]{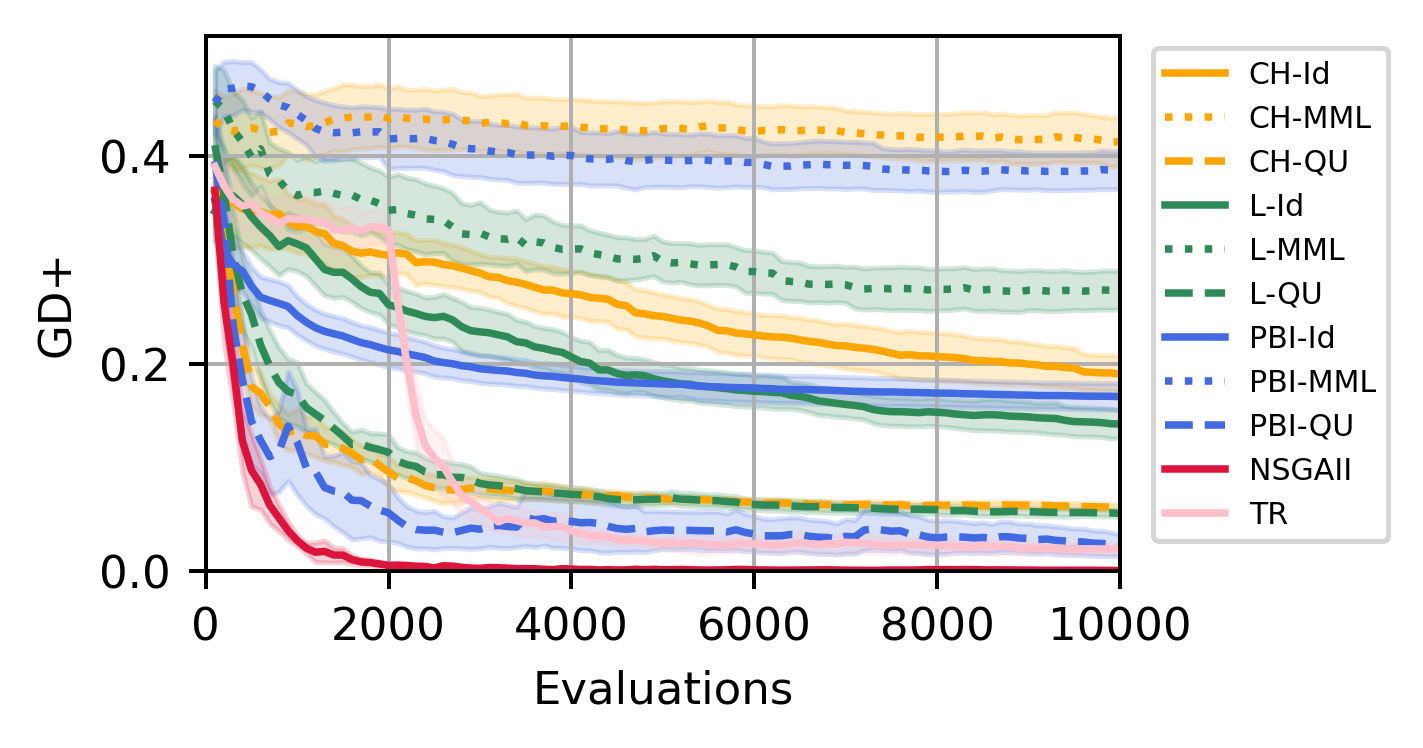}
        \caption{DTLZ 4 ($\GD$)}
        \label{fig:dtlz4-gd}
    \end{subfigure}

    \caption{Comparing HVI and $\GD$ for different optimizers on DTLZ 1-4. As it can be seen in Figure~\ref{fig:dtlz1-hv} and \ref{fig:dtlz3-hv}, Problems 1 and 3 are \emph{unresolved} (i.e., non of the tested optimizers converge to the Pareto-Front).}
    \label{fig:dtlz-plots-part-1}
\end{figure*}

\begin{figure*} %[!t]
    \centering

    \begin{subfigure}{\columnwidth}
        \centering
        \includegraphics[width=\textwidth]{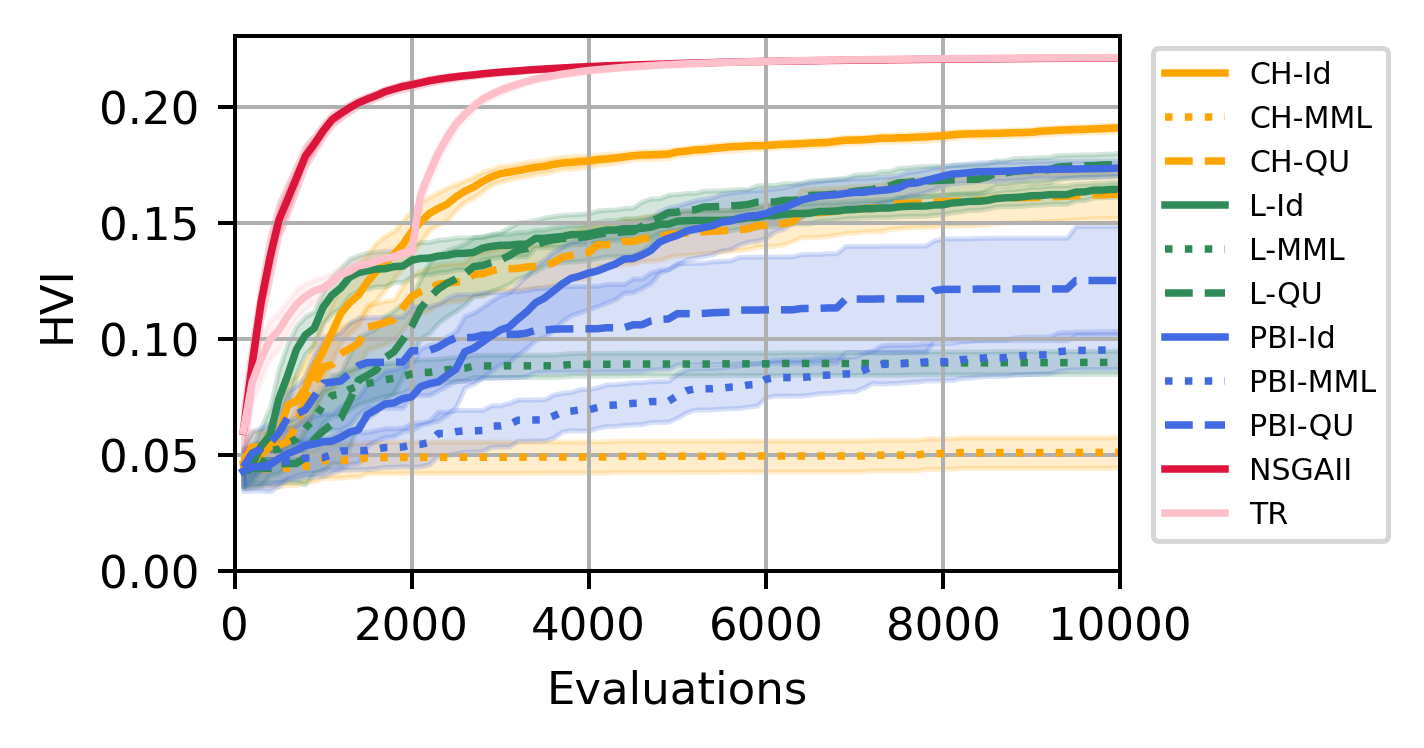}
        \caption{DTLZ 5 (HVI)}
        \label{fig:dtlz5-hv}
    \end{subfigure}
    \begin{subfigure}{\columnwidth}
        \centering
        \includegraphics[width=\textwidth]{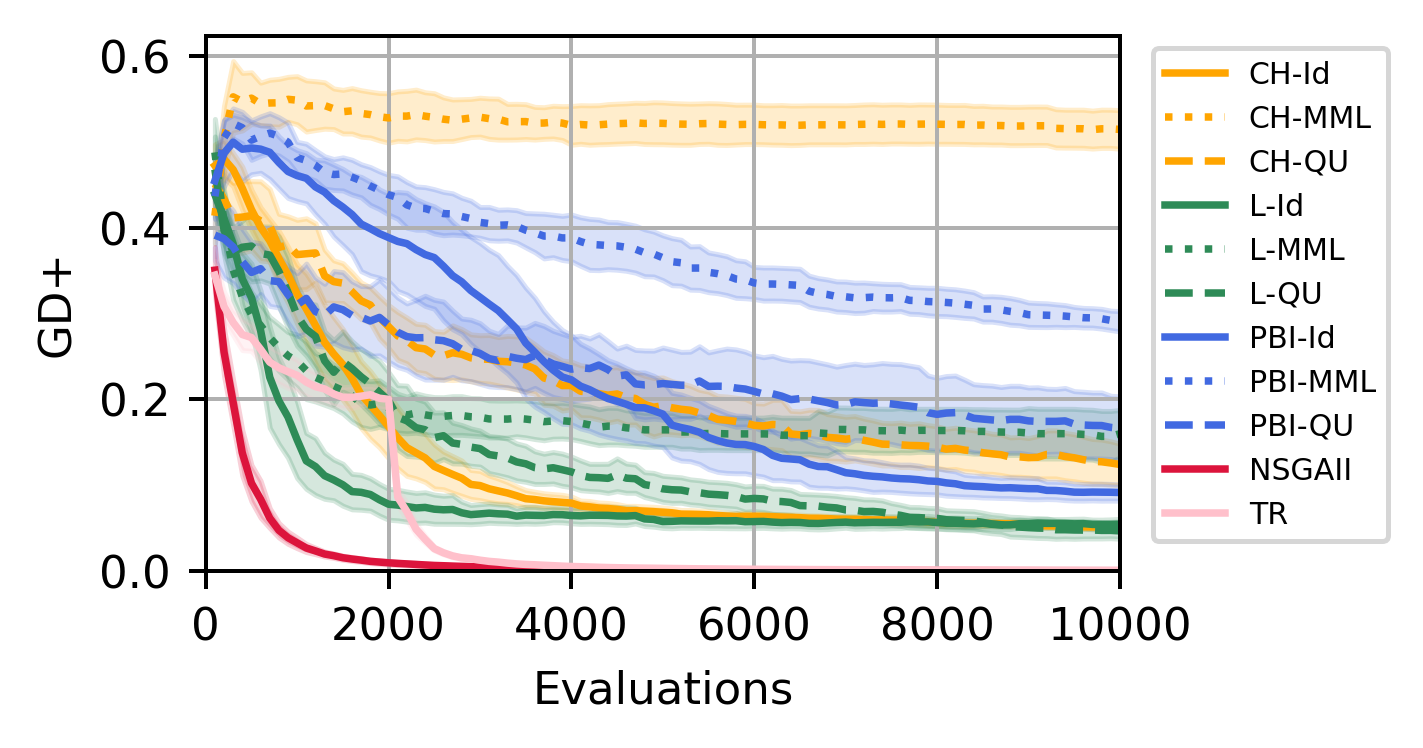}
        \caption{DTLZ 5 ($\GD$)}
        \label{fig:dtlz5-gd}
    \end{subfigure}

    \begin{subfigure}{\columnwidth}
        \centering
        \includegraphics[width=\textwidth]{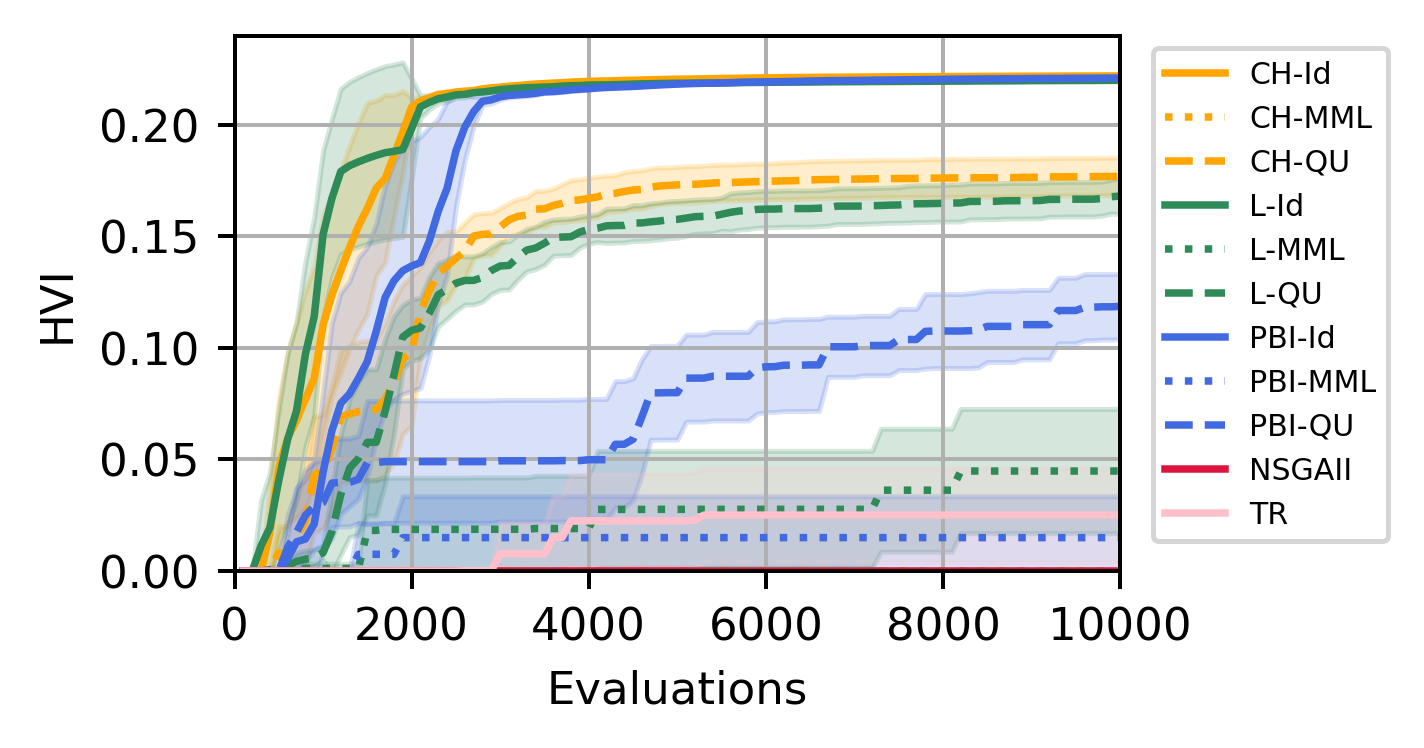}
        \caption{DTLZ 6 (HVI)}
        \label{fig:dtlz6-hv}
    \end{subfigure}
    \begin{subfigure}{\columnwidth}
        \centering
        \includegraphics[width=\textwidth]{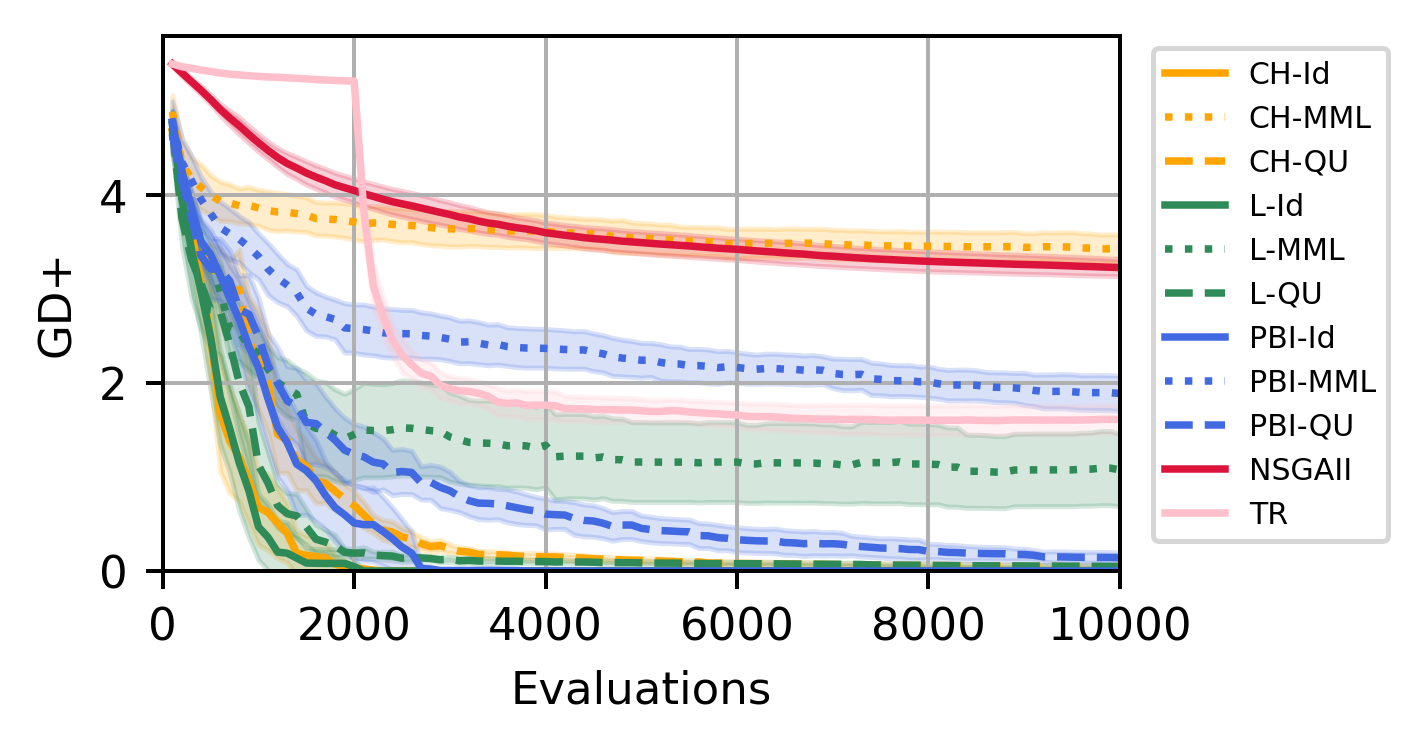}
        \caption{DTLZ 6 ($\GD$)}
        \label{fig:dtlz6-gd}
    \end{subfigure}

    \begin{subfigure}{\columnwidth}
        \centering
        \includegraphics[width=\textwidth]{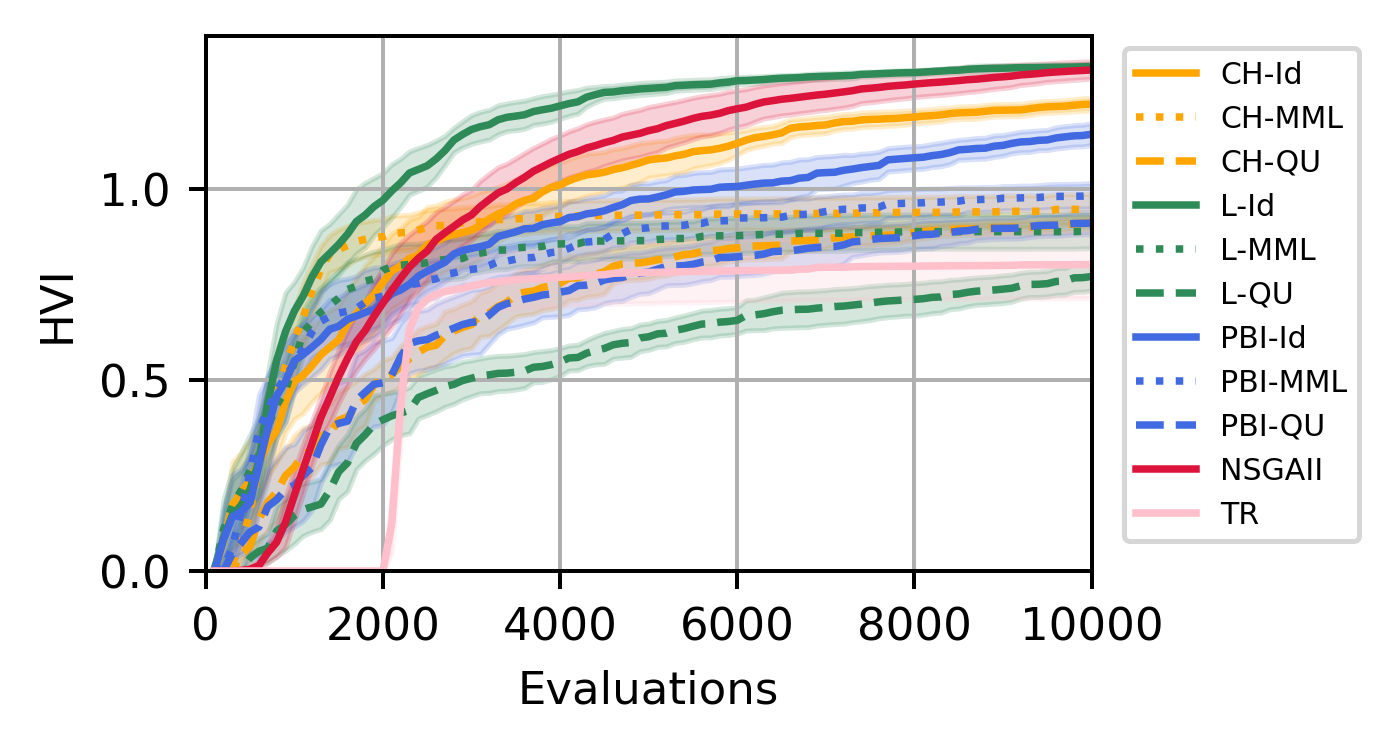}
        \caption{DTLZ 7 (HVI)}
        \label{fig:dtlz7-hv}
    \end{subfigure}
    \begin{subfigure}{\columnwidth}
        \centering
        \includegraphics[width=\textwidth]{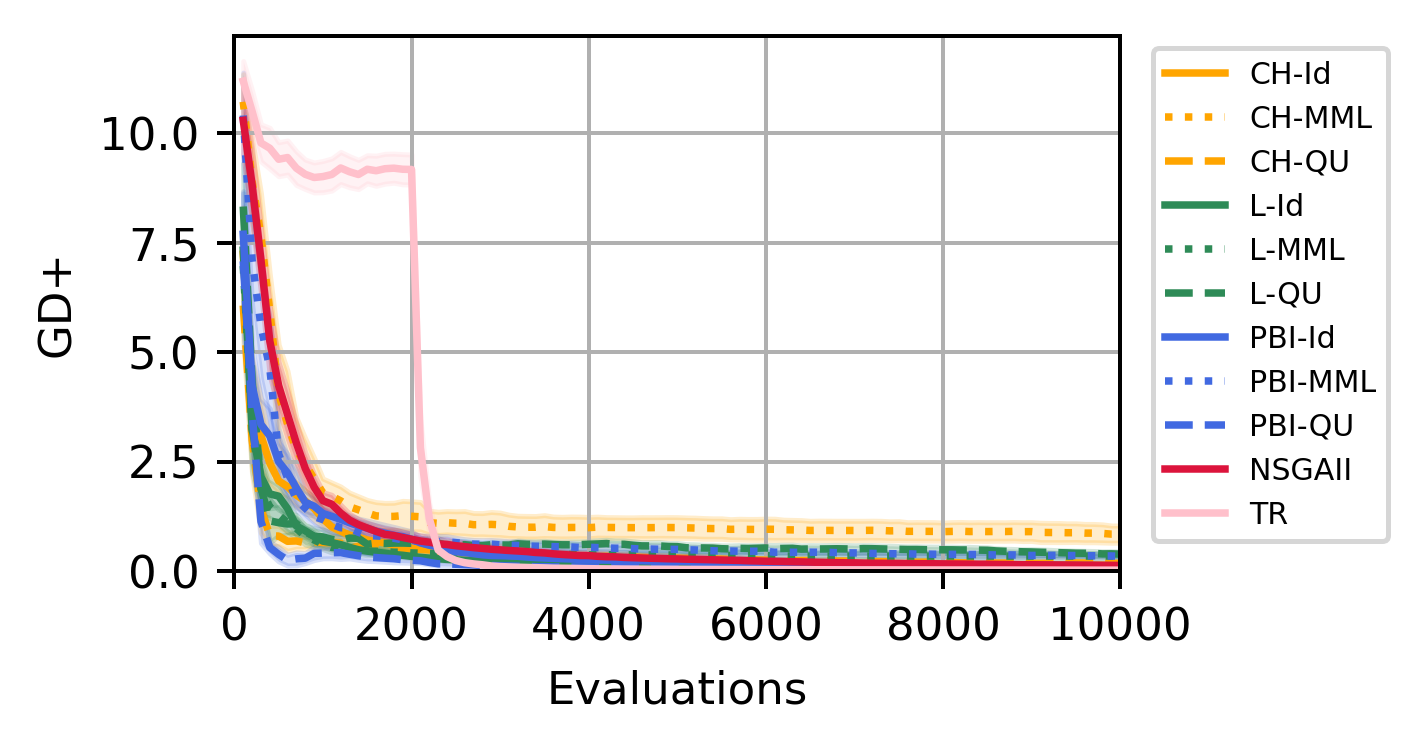}
        \caption{DTLZ 7 ($\GD$)}
        \label{fig:dtlz7-gd}
    \end{subfigure}

    \caption{Comparing HVI and $\GD$ for different optimizers on DTLZ 5-7.}
    \label{fig:dtlz-plots-part-2}
\end{figure*}

\section{Details on the Combo Benchmark}
\label{appendix:combo}

\subsection{Description}

% Text from the previous paper
The Combo benchmark dataset~\cite{xia_predicting_2018} is composed of 165668 training data points (60\%), 55,222 (20\%) validation data points, and 55,222 (20\%) test data points respectively. Each data point has three types of input features: 942 (RNA-Sequence), 3,839 (drug-1 descriptors), and 3,893 (drug-2 descriptors) respectively. The data set size is about 4.2 GB. It is a regression problem with the goal of predicting the growth percentage of cancer cells given cell line molecular features and the descriptors of two drugs. The networks are trained for a maximum budget of 50 epochs and 30 minutes. The validation $R^2$ coefficient at the last trained epoch is used as the objective for hyperparameter search.

The baseline model is composed of 3 inputs each processed by a sub-network of three fully connected layers. Then, the outputs of these sub-models are concatenated and input in another sub-network of 3 layers before the final output. All the fully connected layers have 1000 neurons and ReLU activation. It reaches a validation and test $R^2$ of 0.87 after 100 epochs.
The number of neurons and the activation function of each layer are exposed for the hyperparameter search. The search space is defined as follows: the number of neurons in $[10, 1024]$ with a log-uniform prior; activation function in [elu, gelu, hard sigmoid, linear, relu, selu, sigmoid, softplus, softsign, swish, tanh]; optimizer in [sgd, rmsprop, adagrad, adadelta, adam]; global dropout-rate in $[0, 0.5]$; batch size in $[8, 512]$ with a log-uniform prior; and learning rate in $[10^{-5}, 10^{-2}]$ with a log-uniform prior. A learning rate warmup strategy is activated based on a boolean variable. Accordingly, the base learning rate of this warmup strategy is searched in $[10^{-5}, 10^{-2}]$ with a log-uniform prior. Residual connections are created based on a boolean variable. A learning rate scheduler is activated based on a Boolean variable. The reduction factor of this scheduler is searched in $[0.1, 1.0]$, and its patience in $[5, 20]$. An early-stopping strategy is activated based on a Boolean variable. The patience of this strategy is searched in $[5, 20]$. Then, batch normalization is also activated based on the Boolean variable. The loss is searched among [mse, mae, logcosh, mape, msle, huber]. The data preprocessing is searched among [std, minmax, maxabs]. This corresponds to 22 hyperparameters. All experiments are performed with the same initial random state 42.

\subsection{Other Results}

In this section, we provides additional results from our experiments on the Combo benchmark which strengthen the claims of the study.

%%% About the number of evaluations
First, we start by looking at the parallel performance of the different algorithms. In Table~\ref{tab:table2} we collect the number of evaluations completed successfully (\#D) and the number of evaluations that failed (\#F). It can be observed that Random search samples have a lot more failures than other optimizers. Also, NSGAII and Random search produce more evaluations in total. The differences in the number of completed evaluations (\#D) are purely due to the difference in strategy between the different optimizers. In fact, we computed the \emph{utilization} versus time in Figure~\ref{fig:combo-utilization}. 
The utilization at a time $t$ corresponds to the ratio $U(t)= W_a(t) / W$ of active worker $W_a$ out of all available workers $W$. Therefore, we can see in Figure~\ref{fig:combo-utilization} that the workers remain busy more than 95\% of the time which means that fewer evaluations do not come from a lack of efficiency of the optimizer to cope with the demand. In addition, having all optimizers with good utilization also means that the difference in observed HVI is not a by-product of parallel optimization such as in~\cite{egele2022asynchronous} (which improve utilization) but a result of differences in multi-objective optimization strategies. For NSGAII which performs more evaluation than D-MoBO, this can be justified by a policy with stronger exploitation for NSGAII. As one of the objectives ($y_2$) is latency and the other is the number of parameters $y_3$, both are highly correlated to training speed, NSGAII will favor configurations training faster (because it ranks candidates by successive non-dominance). On the other side, D-MoBO which is based on BO has more exploration (led by its uncertainty) even if Algorithm~\ref{alg:dbo-algorithm} reduced it through the periodic decay (line 9).

\begin{table} % [!b]
    \centering
    \resizebox{\columnwidth}{!}{%
    \begin{tabular}{l|cc|cc|cc}
        & \multicolumn{2}{c|}{\textbf{40 Workers}} & \multicolumn{2}{c|}{\textbf{160 Workers}} & \multicolumn{2}{c}{\textbf{640 Workers}} \\
        & \#D   & \#F    & \#D     & \#F      & \#D     & \#F     \\ \hline
    \textbf{Random}    & 570   & 8     & 2,462   & 51     & 10,512  & 227   \\
    \textbf{MoTPE}     & 253   & 12    & 1,328   & 2      & 6,912   & 13    \\
    \textbf{NSGAII}    & 543   & 4     & 2,396   & 8      & 11,957  & 15    \\
    \textbf{D-MoBO} &
    496 &
    0 &
    1,835 &
    5 &
    7,842 &
    14  \\
    \textbf{D-MoBO NP} & 307   & 2      & 1353    & 9       & 4,947   & 24   
    \end{tabular}%
    }
    \caption{Comparing optimizers different numbers of workers. \#D: the number of completed evaluations and \#F: the number of failed evaluations.}
    \label{tab:table2}
\end{table}

\begin{figure*}[!t]
    \centering
    \begin{subfigure}{\columnwidth}
        \centering
        \includegraphics[width=0.9\textwidth]{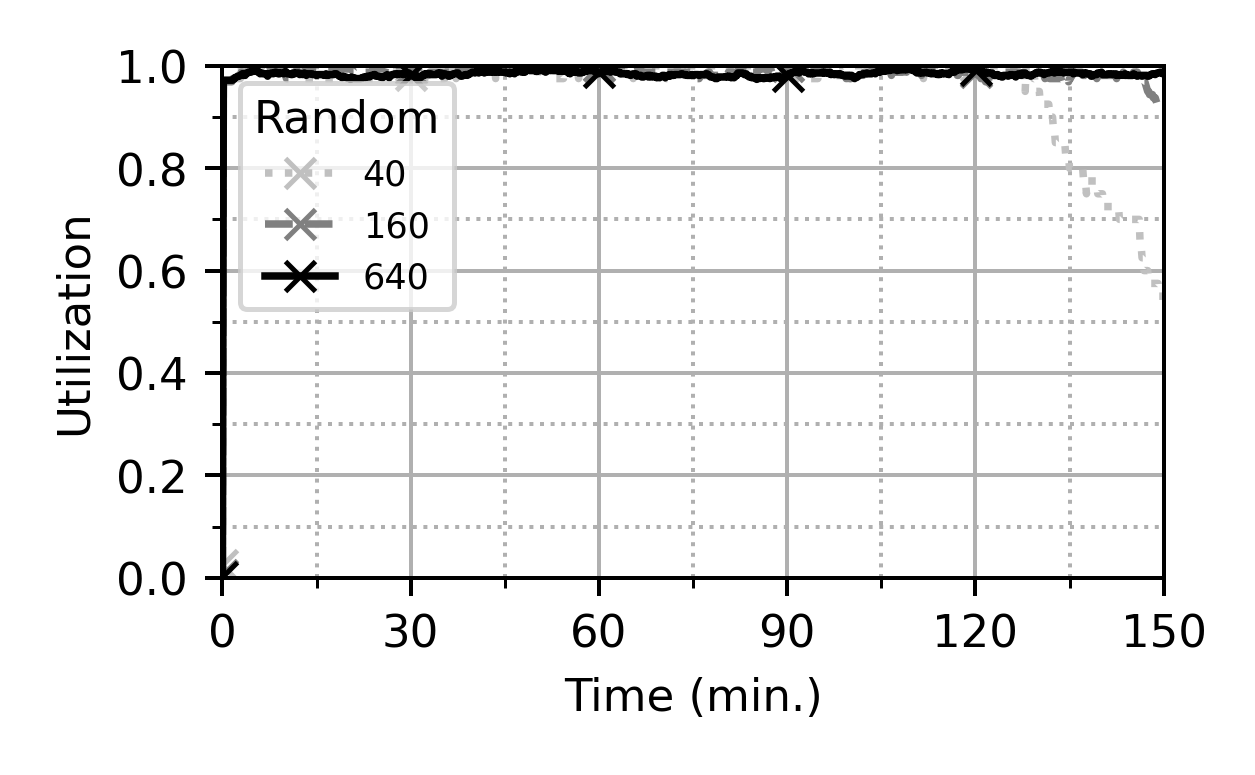}
        \caption{Random}
        \label{fig:utilization-vs-time-polaris-combo-random-scaling.png}
    \end{subfigure}
    \begin{subfigure}{\columnwidth}
        \centering
        \includegraphics[width=0.9\textwidth]{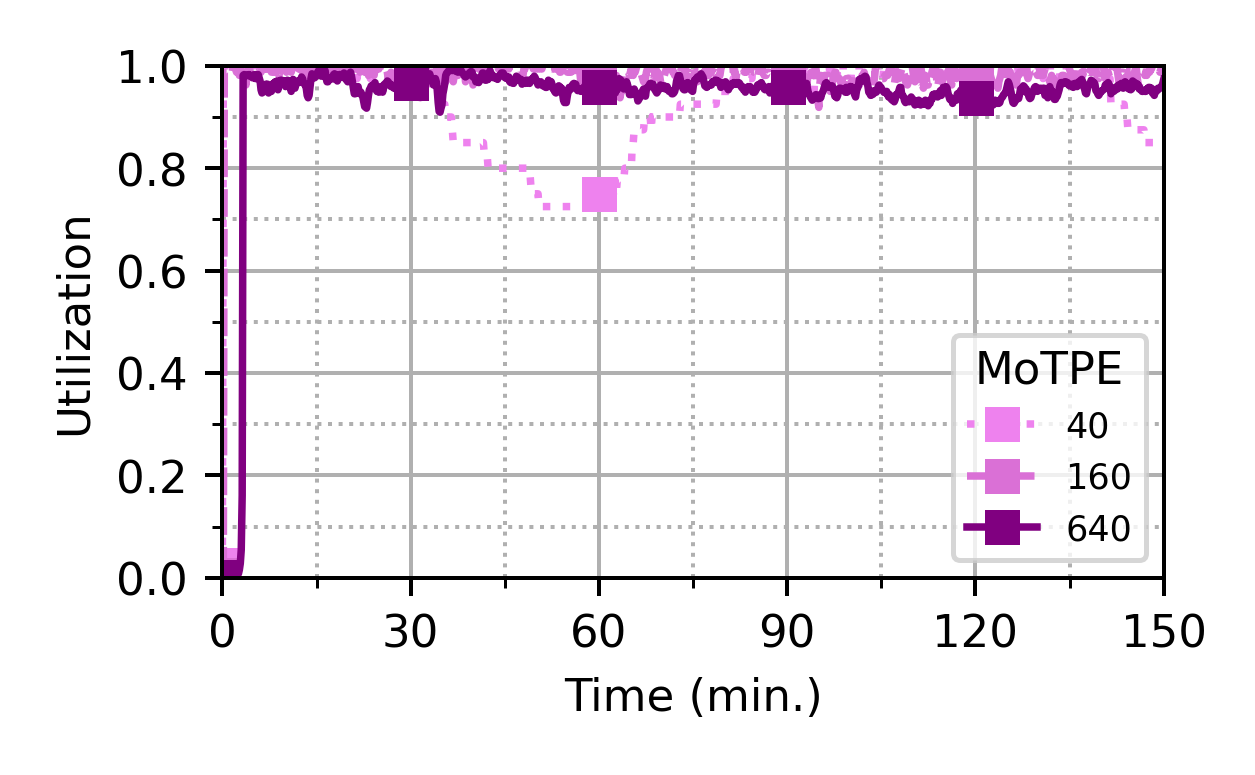}
        \caption{MoTPE}
        \label{fig:utilization-vs-time-polaris-combo-motpe-scaling.png}
    \end{subfigure}
    \begin{subfigure}{\columnwidth}
        \centering
        \includegraphics[width=0.9\textwidth]{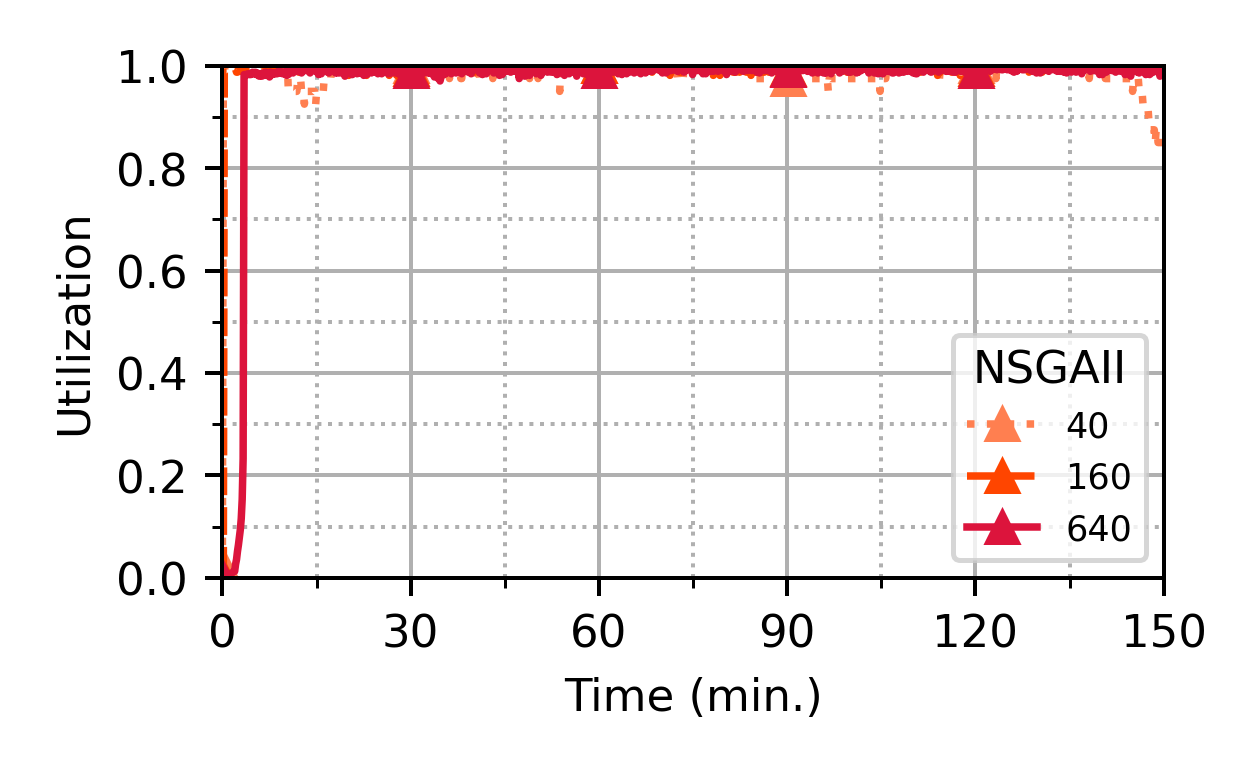}
        \caption{NSGAII}
        \label{fig:utilization-vs-time-polaris-combo-nsgaii-scaling.png}
    \end{subfigure}
    \begin{subfigure}{\columnwidth}
        \centering
        \includegraphics[width=0.9\textwidth]{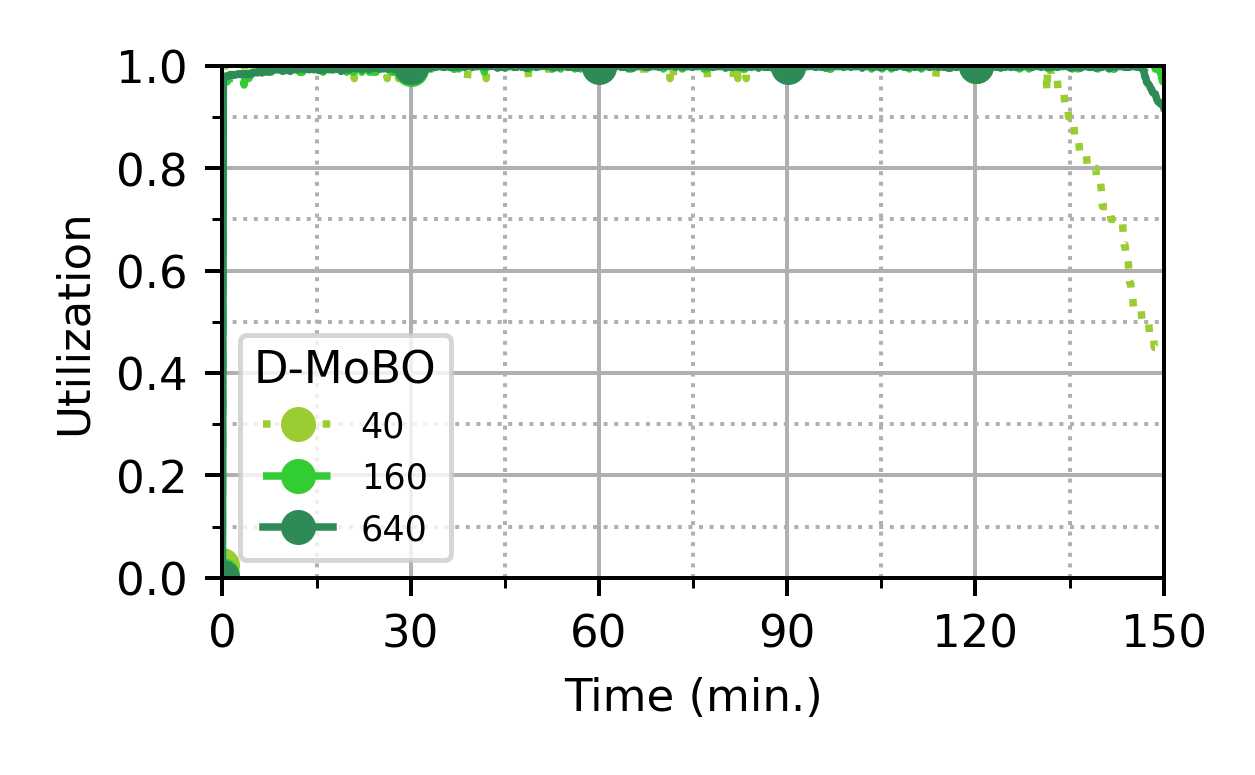}
        \caption{D-MoBO}
        \label{fig:utilization-vs-time-polaris-combo-d-mobo-scaling.png}
    \end{subfigure}
    \caption{Comparing worker utilization of MOO algorithms w.r.t. time in minutes for different number of parallel workers.}
    \label{fig:combo-utilization}
\end{figure*}

%%% The effect of penalty at different scales
Second, we look at the effect of the penalty for different numbers of workers in order to evaluate if its effect is more or less ``independent'' from the scale. As it can be observed in Figure~\ref{fig:combo/hypervolume-vs-time-penalty} for 40, 160, and 640 workers the HVI curve of D-MoBO (implementing the penalty) is always dominating the HVI curve of D-MoBO (NP) (without penalty). This confirms that the penalty is efficient to give more weight to trade-offs of interest during the optimization.

\begin{figure}[!t]
    \centering
    \begin{subfigure}{\columnwidth}
        \centering
        \includegraphics[width=0.9\textwidth]{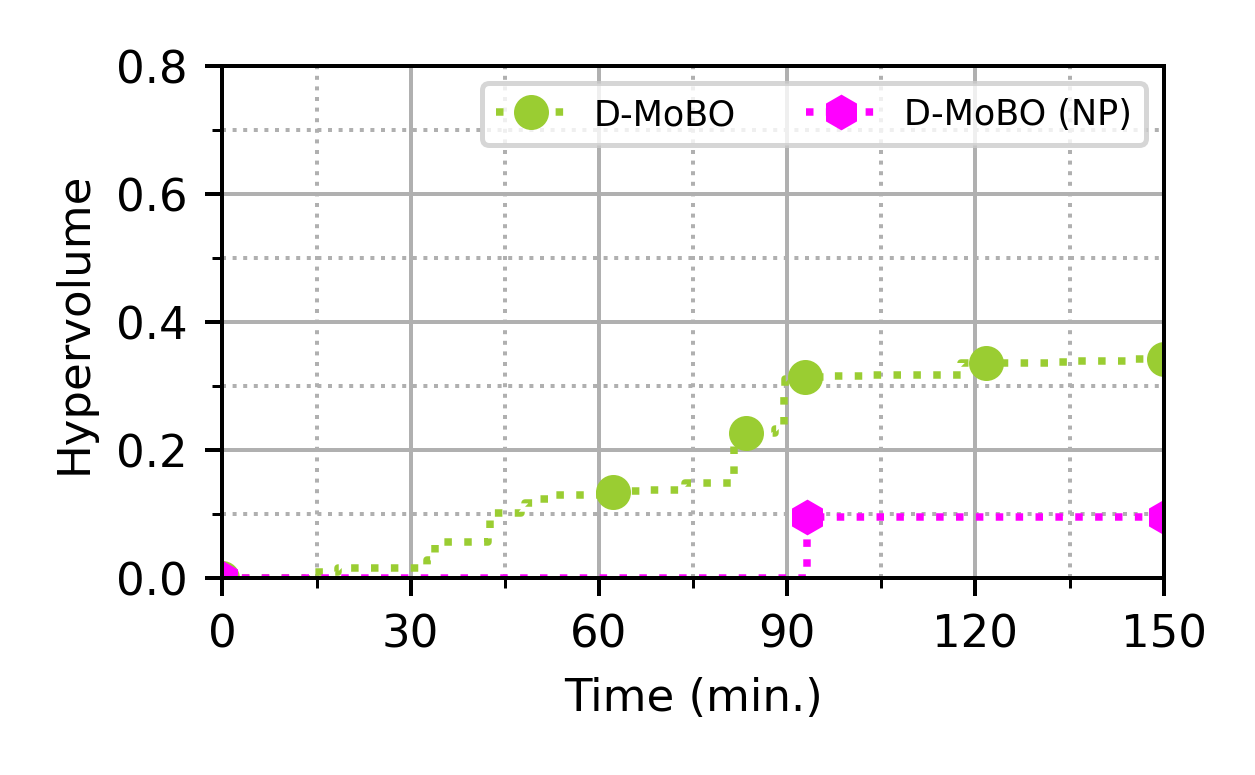}
        \caption{40 Workers}
        \label{fig:hypervolume-vs-time-polaris-combo-d-mobo-constraint-10}
    \end{subfigure}
    \begin{subfigure}{\columnwidth}
        \centering
        \includegraphics[width=0.9\textwidth]{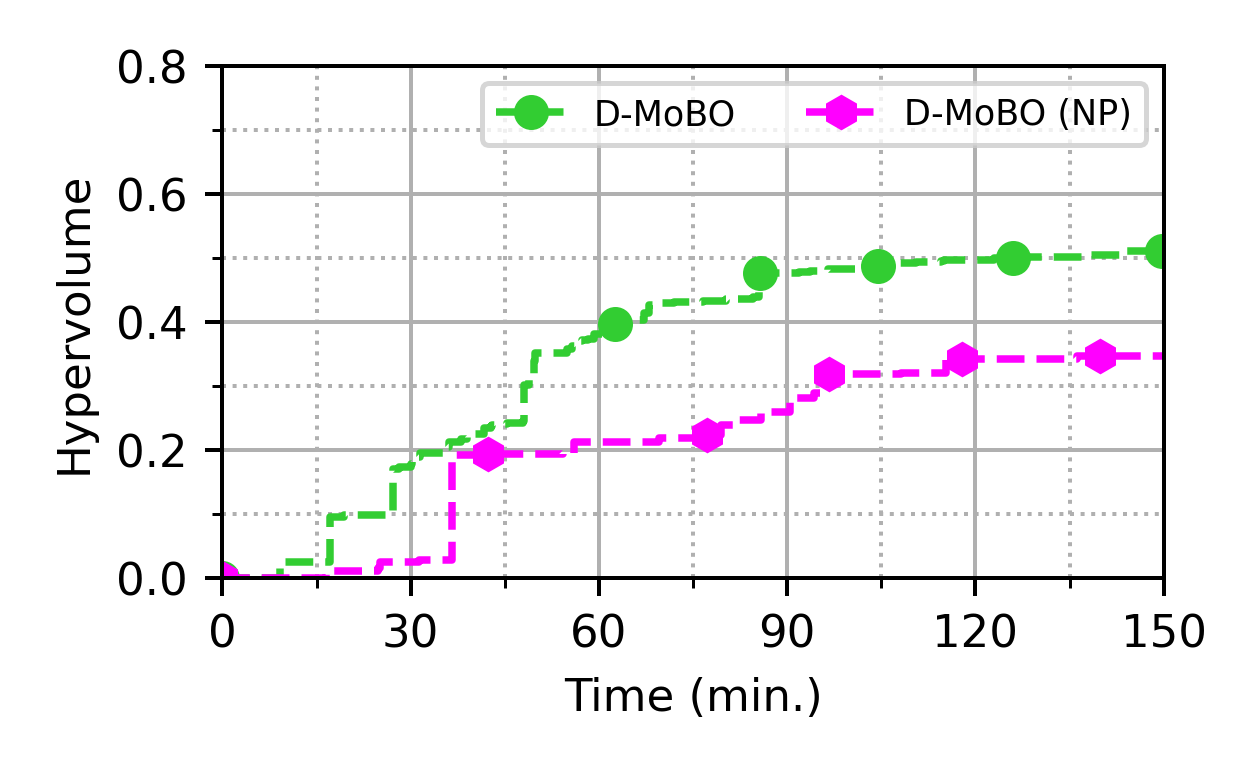}
        \caption{160 Workers}
        \label{fig:hypervolume-vs-time-polaris-combo-d-mobo-constraint-40}
    \end{subfigure}
    \begin{subfigure}{\columnwidth}
        \centering
        \includegraphics[width=0.9\textwidth]{figures/combo/hypervolume-vs-time-polaris-combo-d-mobo-constraint-160.png}
        \caption{640 Workers}
        \label{fig:hypervolume-vs-time-polaris-combo-d-mobo-constraint-160}
    \end{subfigure}
    \caption{Comparing the effect of D-MoBO (with a penalty to enforce the minimum accuracy requirement) and D-MoBO (NP) (without penalty) for different numbers of parallel workers (40, 160, and 640).}
    \label{fig:combo/hypervolume-vs-time-penalty}
\end{figure}

%%% The order of best performers for MOO remains independantly from the scale
Third, we look at the performance of the optimizers for different numbers of parallel workers in order to evaluate if the order of best performers remains the same. As it can be observed in Figure~\ref{fig:combo/hypervolume-vs-time} for 40, 160, and 640 workers the order of best performers remains the same. This confirms that the conclusions of the study are not a by-product of the scale of the problem.

\begin{figure}[!t]
    \centering
    \begin{subfigure}{\columnwidth}
        \centering
        \includegraphics[width=0.8\textwidth]{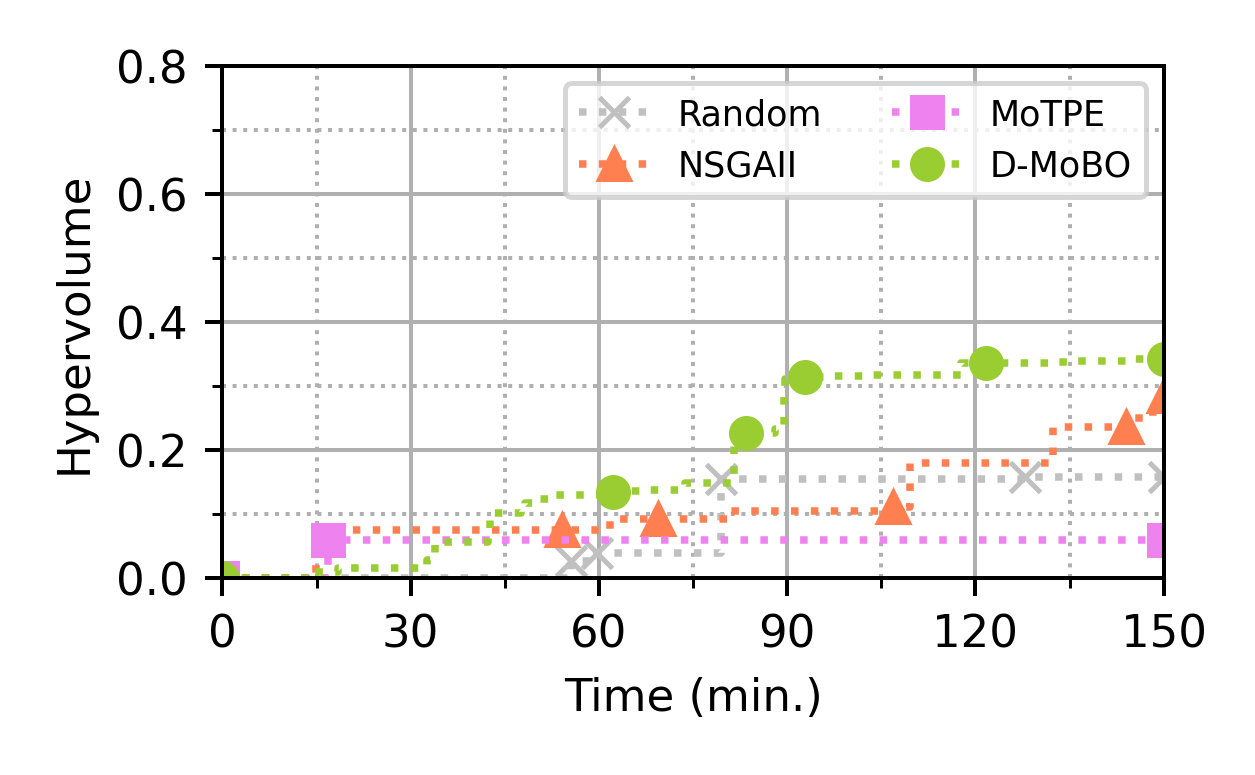}
        \caption{40 Workers}
        \label{fig:combo/hypervolume-vs-time-polaris-combo-all-40}
    \end{subfigure}
    \begin{subfigure}{\columnwidth}
        \centering
        \includegraphics[width=0.9\textwidth]{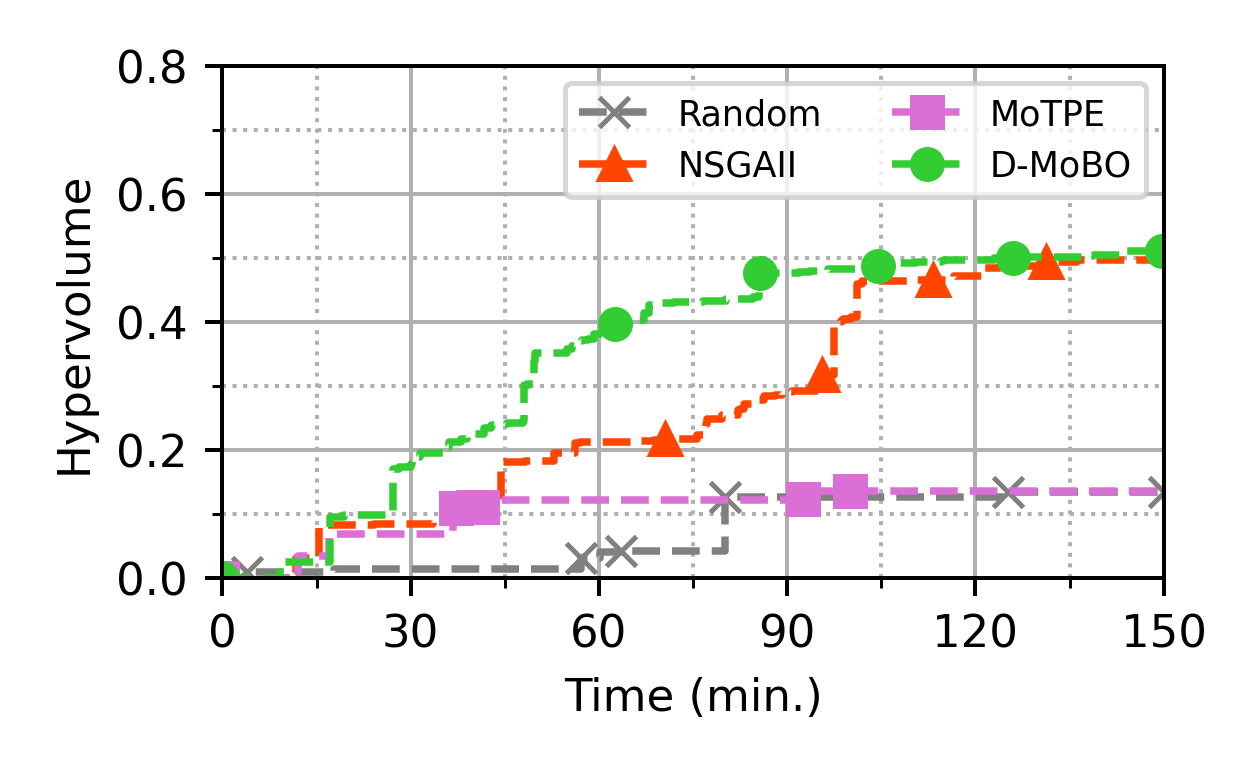}
        \caption{160 Workers}
        \label{fig:combo/hypervolume-vs-time-polaris-combo-all-160}
    \end{subfigure}
    \begin{subfigure}{\columnwidth}
        \centering
        \includegraphics[width=0.9\textwidth]{figures/combo/hypervolume-vs-time-polaris-combo-all-160.png}
        \caption{640 Workers}
        \label{fig:combo/hypervolume-vs-time-polaris-combo-all-640}
    \end{subfigure}
    \caption{Comparing the HVI curves of optimizers for different numbers of parallel workers (40, 160, and 640).}
    \label{fig:combo/hypervolume-vs-time}
\end{figure}

%%% The effect of the number of workers on the performance of MOO
Finally, we look at the gain optimization performance of all considered optimizers when increasing the number of parallel workers. As it can be observed in Figure~\ref{fig:combo/hypervolume-scaling} NGAII and D-MoBO benefit significantly from increasing the number of workers. This shows that there exists a synergy effect between the optimization algorithm and parallel scale. Not all optimization strategies will benefit from parallelism in the same way. For some, it can be wasteful to increase the scale as the outcome will barely change.

\begin{figure*}[!t]
    \centering
    \begin{subfigure}{\columnwidth}
        \centering
        \includegraphics[width=0.9\textwidth]{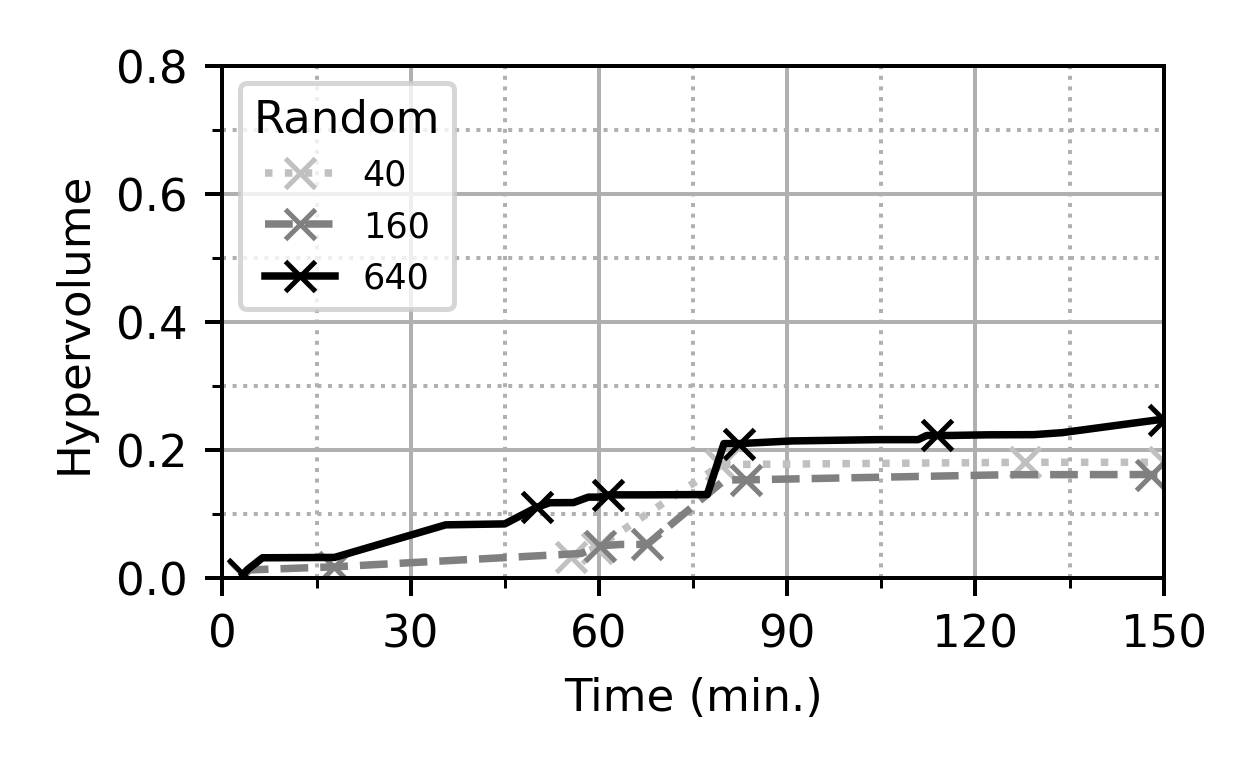}
        \caption{Random}
        \label{fig:combo/hypervolume-scaling/random}
    \end{subfigure}
    \begin{subfigure}{\columnwidth}
        \centering
        \includegraphics[width=0.9\textwidth]{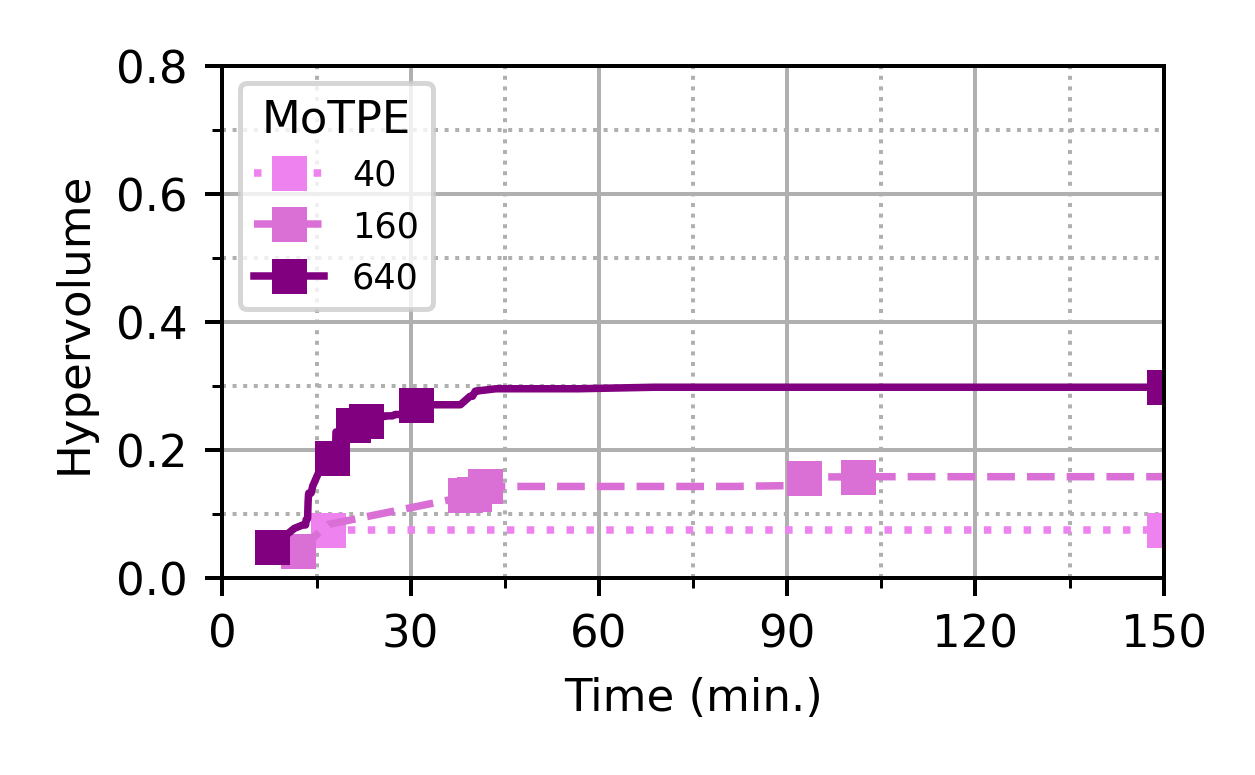}
        \caption{MoTPE}
        \label{fig:combo/hypervolume-scaling/motpe}
    \end{subfigure}
    \begin{subfigure}{\columnwidth}
        \centering
        \includegraphics[width=0.9\textwidth]{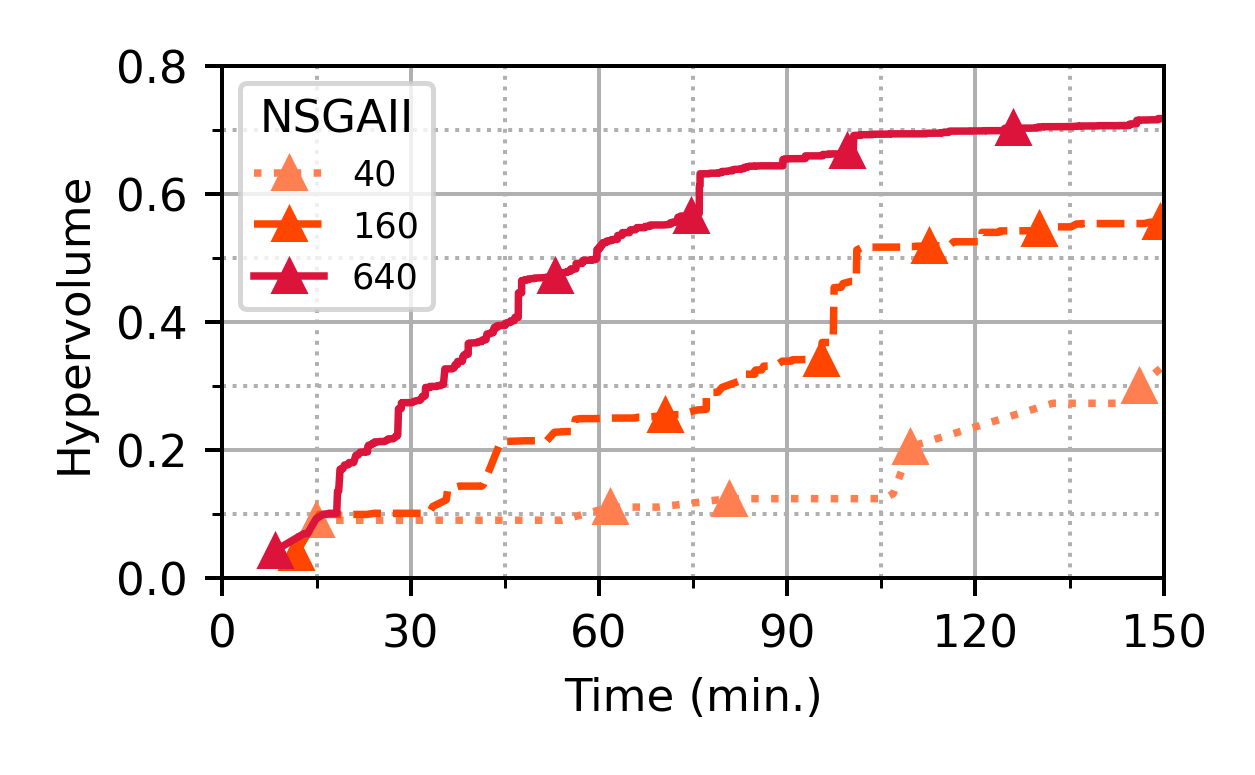}
        \caption{NSGAII}
        \label{fig:combo/hypervolume-scaling/nsgaii}
    \end{subfigure}
    \begin{subfigure}{\columnwidth}
        \centering
        \includegraphics[width=0.9\textwidth]{figures/combo/hypervolume-vs-time-polaris-combo-d-mobo-scaling.png}
        \caption{D-MoBO}
        \label{fig:combo/hypervolume-scaling/d-mobo}
    \end{subfigure}
    \caption{Comparing HVI of MOO algorithms w.r.t. time in minutes for different numbers of parallel workers.}
    \label{fig:combo/hypervolume-scaling}
\end{figure*}

\section{Description of HPC System}
\label{sec:polaris-description}

In this section we provide details on the HPC system used for our experiments on the Combo benchmark.
% Text from previous paper
They were conducted on the Polaris supercomputer at the Argonne Leadership Computing Facility (ALCF). Polaris is a HPE Apollo Gen10+ platform that comprises 560 nodes, each  equipped with a 32-core AMD EPYC "Milan" processor, 4 Nvidia A100 GPUs, and 512 GB of DDR4 memory. The compute nodes are interconnected by a Slingshot network. In our experiments, each worker is attributed 1 GPU. The algorithm is implemented in Python 3.8.13 where the main packages used are deephyper 0.5.0, mpi4py 3.1.3, scikit-learn 1.0.2,  optuna 3.1.0, Redis 7.0.5. Neural networks are implemented in Tensorflow 2.10. The number of workers is increased from 40 (10 nodes), 160 (40 nodes) to 640 (160 nodes).

\end{document}